\documentclass[12pt]{l4dc2023}

\title[Efficient Reinforcement Learning Through Trajectory Generation]{Efficient Reinforcement Learning Through Trajectory Generation}
\usepackage{times}

\graphicspath{{figure/}}
\usepackage{float}
\usepackage{amsmath,amssymb,bm}
\usepackage{xcolor}
\usepackage{array}
\usepackage{cite}

\usepackage{algorithm,algorithmic}

\usepackage{bookmark}
\usepackage{color}

\newcommand{\bigzero}{\mbox{\normalfont\Large\bfseries 0}}
\newcommand{\rvline}{\hspace*{-\arraycolsep}\vline\hspace*{-\arraycolsep}}

\usepackage{bm}
\usepackage{bbm}
\usepackage{algorithmic}
\usepackage{algorithm}
\SetKwComment{Comment}{$\triangleright$\ }{}

\usepackage{array}

\usepackage{amssymb}
\usepackage{dsfont}
\usepackage{mathtools}
\usepackage{siunitx}

\usepackage{enumitem}

\DeclareSIUnit[]{\pu}{p.u.}
\DeclareSIUnit[]{\VA}{VA}

\DeclareSymbolFont{bbold}{U}{bbold}{m}{n}
\DeclareSymbolFontAlphabet{\mathbbold}{bbold}

\DeclarePairedDelimiterX\Set[2]{\lbrace}{\rbrace}%
{ #1 \,\delimsize| \,\mathopen{} #2 }

\newcommand{\real}[0]{\mathbb R}

\usepackage{tikz}

\usepackage{xcolor}

\author{%
 \Name{Wenqi Cui} \Email{wenqicui@uw.edu}\\
 \addr Department of Electrical and Computer Engineering, University of Washington, WA, USA	 
 \AND
 \Name{Linbin Huang} \Email{linhuang@ethz.ch }\\
 \addr Department of Information Technology and Electrical Engineering, ETH Zurich, Switzerland
  \AND
 \Name{Weiwei Yang} \Email{weiwya@microsoft.com }\\
 \addr Microsoft Research, Redmond, WA, USA%
   \AND
 \Name{Baosen Zhang} \Email{zhangbao@uw.edu}\\
 \addr Department of Electrical and Computer Engineering, University of Washington, WA, USA	 
}

\begin{document}

\maketitle

\begin{abstract}%

A key barrier to using reinforcement learning (RL) in many real-world applications is the requirement of a large number of system interactions to learn a good control policy. Off-policy and Offline RL methods have been proposed to reduce the number of interactions with the physical environment by learning control policies from historical data. However, their performances suffer from the lack of exploration and the distributional shifts in trajectories once controllers are updated. Moreover, most RL methods require that all states are directly observed, which is difficult to be attained in many settings.

To overcome these challenges, we propose a trajectory generation algorithm, which adaptively generates new trajectories as if the system is being operated and explored under the updated control policies. Motivated by the fundamental lemma for linear systems, assuming sufficient excitation, we generate trajectories from linear combinations of historical trajectories. For linear feedback control, we prove that the algorithm generates trajectories with the exact distribution as if they were sampled from the real system using the updated control policy. In particular, the algorithm extends to systems where the states are not directly observed. Experiments show that the proposed method significantly reduces the number of sampled data needed for RL algorithms. 

\end{abstract}


\section{Introduction}

Reinforcement learning (RL) is becoming increasingly popular for the controller design of dynamical systems, especially when the exact system model or parameters are not available~\citep{zheng2021sample,hu2022towards}. Much of the success in RL has relied on sampling-based  algorithms such as the policy gradient algorithm~\citep{Sutton2018RL, fazel2018global}, which typically requires repeated online interactions with the system.  Moreover, the control actions need to incorporate sufficient exploration for the learning algorithm to search for better policies~\citep{jin2021pessimism}. However, sampling large batch of trajectories is expensive in many real-world problems (e.g., in energy systems, robotics or healthcare), and the exploration requirement  for safety-critical  systems may be dangerous~\citep{levine2020offline, fujimoto2021minimalist}.

When the online interactions with the system are limited, two categories of RL methods are designed: off-policy RL and offline RL. 
Off-policy RL methods (e.g., Q-learning and its variants) typically learn a quality function (i.e, Q function)  leveraging  past  experience, but online interactions with explorations are still required after the update of the control policies~\citep{ghasemipour2021emaq}.
% where the  control policy at the current step is restricted to be close enough from that in the previous iteration step. However, this algorithm can only reuse the historic trajectory from the previous iteration step, and exploration and sampling still need to be done on the true system using the latest policy.
Offline RL seeks to learn from a
fixed dataset without interactions with environments~\citep{gulcehre2020rl, jin2021pessimism}. The fundamental challenge is that once the control policies have been updated, the trajectories of the system under the new policies would not have the same distribution as the historical data~\citep{fujimoto2021minimalist, ghasemipour2021emaq}. As a result,  existing algorithms typically constrain the control policy to be close with the policy utilized in the fixed dataset~\citep{ ostrovski2021difficulty, fujimoto2021minimalist}. 
Since most algorithms need to do some exploration, it is believed that past data is not helpful if high-reward regions are not covered in the  collected trajectories~\citep{levine2020offline, jin2021pessimism}. 
% if useful for explorations if the 
% Moreover, it is challenging to develop rigorous performance guarantees on the trained controllers using past trajectories. 
% Most offline RL methods also require that  all system states can be directly observed, which is difficult to be attained for many applications.

A fundamental reason behind the above challenges is that the training process is restricted to fixed trajectories in the historical data, hence RL algorithms need to be restricted to historical control policies. We look at the problem from the other direction: \textit{Using only historical data, can we generate trajectories that follow the same distribution induced by a new control policy?}

% To answer this question, we need to build the mapping between past trajectories to all the possibly new trajectories under updated control policy. 

This paper proposes a trajectory generation algorithm for linear systems, which adaptively generates new trajectories as if the system were being operated and explored under the updated control policies.
The key insights come from  the fundamental lemma for linear systems, which shows that any set of persistently exciting trajectories can be used to represent the input-output behavior of the system~\citep{willems2005note, de2019formulas, markovsky2022identifiability}. Inspired by this, we generate trajectories from linear combinations of historical trajectories, which can come from routine operations of the system.  The set of linear combinations is derived from  the updated control policy with perturbations on actions, such that  the generated trajectory is the same as the trajectory sampled on the real system. This adaptive approach overcomes the challenges in distributional shift and lack of exploration. This is complementary to recent advances in learning linear feedback controllers for linear systems (See, for example,~\citep{fazel2018global,zheng2021sample,tang2021analysis,hu2022towards} and references within), where  trajectories are sampled through online interactions.  
Experiments
% in the control of  batch reactor systems and power distribution systems
show that the proposed method significantly reduces the number of sampled data needed for RL algorithms. 
We summarize the contributions as follows:
\begin{enumerate}[noitemsep,nolistsep,leftmargin=*, label=\arabic*)]
    \item We propose a  simple end-to-end approach to generate input-output trajectories for linear systems, which significantly reduces the burden of sample collection in RL methods. In Theorem~\ref{thm: unque_trajectory0}, 
    we prove that the generated trajectories is adaptive to the distribution shift of any linear state-feedback controller with perturbations on actions for explorations. When the states are not directly observed, Theorem~\ref{thm: unque_trajectory0_partial} shows that this framework also applies to output-feedback control by defining an extended state from the observations.

%   \item  When  the states are not observed, the trajectory generation algorithm also directly extends to output-feedback control by equivalently defining system transition dynamics using trajectories of observations.  In Theorem~\ref{thm: unque_trajectory0_partial}, we prove that  we prove that the generated trajectories are identical to the trajectories under any linear output-feedback control with noises for exploration. 
  \item The proposed trajectory generation algorithm is compatible with any RL methods that learns from trajectories. The number of samples needed to learn is independent to the batch size (the number of trajectories in each episode) and the number of training episodes. 
  % greatly increase efficiency of RL methods.   
\end{enumerate}

\section{ Preliminaries and problem formulation}
\label{sec: background}
% \vspace{-0.1cm}

\subsection{Notations}

Throughout this manuscript, vectors are denoted in lower case bold and matrices are denoted in upper case bold, unless otherwise specified. Vectors of all ones and zeros are denoted by $\mathbbold{1}_n, \mathbbold{0}_n \in \real^n$, respectively. The identity matrix  is denoted by $\bm{I}_n\in\real^{n\times n}$. 
% \bz{none of the matrices or vectors are bold though}
% We use $\mathbb{N}, \mathbb{R}, \mathbb{R}_{\geq 0}$ and $\mathbb{R}_{>0}$ to denote the set of natural, real, nonnegative and positive real numbers, respectively. 
% For $A \in \mathbb{R}^{m \times n},[A]_i$ and $[A]_{i j}$ represent its $i$-th row and $(i, j)$-th element, respectively. 
We use $\mathcal{N}(\bm{A})$ to denote the null space of matrix $\bm{A}$. Given $n$ matrices $\bm{M}_i, i=1,\cdots,n$, we denote $[\bm{M}_1;\cdots;\bm{M}_n]:=[\bm{M}_1^\top \cdots \bm{M}_n^\top]^\top$.
% A continuous function $\alpha: \mathbb{R} \rightarrow \mathbb{R}$ is of (extended) class- $\mathcal{K}$ if it is strictly increasing and $\alpha(0)=0$. 

Given a discrete-time signal   $\bm{z}(t)\in \mathbb{R}^d$ for $t=0,1,\cdots$, we 
% adopt the notation
% in~\citep{de2019formulas} 
use $\bm{z}_{[k, k+T]}\in\real^{Td}$  to denote the vector form of the sequence $\{\bm{z}(k), \ldots, \bm{z}(k+T)\}$,
% for $k \in \mathbb{Z}, T \in \mathbb{N}$, 
and the Hankel matrix $\bm{Z}_{i, t, N}\in\real^{td\times N}$ as
% associated to the signal $\bm{z}$ as follows
\vspace{-0.1cm}
$$
\bm{Z}_{[k, k+T]}=\left[\begin{array}{c}
\bm{z}(k) \\
\vdots \\
\bm{z}(k+T)
\end{array}\right],
\bm{Z}_{i, t, N}=\left[\begin{array}{cccc}
\bm{z}(i) & \bm{z}(i+1) & \cdots & \bm{z}(i+N-1) \\
\bm{z}(i+1) & \bm{z}(i+2) & \cdots & \bm{z}(i+N) \\
\vdots & \vdots & \ddots & \vdots \\
\bm{z}(i+t-1) & \bm{z}(i+t) & \cdots & \bm{z}(i+t+N-2)
\end{array}\right],
$$
where $k,i$ and are integers, and $t, N, T $ are natural numbers. The first subscript of $\bm{Z}_{i, t, N}$ denotes the time at which the first sample of the signal is taken, the second one the number of samples per each column, and the last one the number of signal samples per each row. 
% \lh{not necessary to cite it here because you use the same notation system}
% Sometimes, if $t=1$, noting that the matrix $\bm{z}_{i, t, N}$ has only one block row, we simply write
% $$
% \bm{z}_{i, N}=\left[\begin{array}{llll}
% \bm{z}(i) & \bm{z}(i+1) & \cdots & \bm{z}(i+N-1)
% \end{array}\right] .
% $$

\subsection{Problem formulation}\label{subsec: problem}
We consider a  discrete-time linear time-invariant system
\vspace{-0.1cm}
\begin{equation}\label{eq:sys_full}
\begin{aligned}
\bm{x}(k+1) &=\bm{A} \bm{x}(k)+\bm{B} \bm{u}(k) \\
\bm{y}(k) &=\bm{C} \bm{x}(k)
\end{aligned}
\end{equation}
with state $\bm{x} \in \mathbb{R}^n$, action $\bm{u} \in \mathbb{R}^m$, and output  $\bm{y} \in \mathbb{R}^q$ (sometimes called observations in RL literature).  If $\bm{C}$ is not of full (column) rank, we say that the states are not directly observed. Otherwise, we assume $\bm{C}=\bm{I}_n$ and the states are directly observed.  We assume that $\bm{A}$ and $\bm{B}$ are not known. The matrix $\bm{C}$ is also unknown if the state is not directly observed. We also make the standard assumption that $\left(\bm{A},\bm{B}\right)$ is stabilizable and $\left(\bm{A},\bm{C}\right)$ observable~\citep{hespanha2009linear}. 
% \vspace{-0.1cm}
% \begin{assumption}
%  $\left(\bm{A},\bm{B}\right)$ is stabilizable and $\left(\bm{A},\bm{C}\right)$ observable. 
% \end{assumption}

A trajectory is a sequence of observations and actions of length $T$, given by $\bm{\tau}=\{\bm{y}(0), \bm{u}(0), \cdots \\ ,  \bm{y}(T-1), \bm{u}(T-1)\}$. The control action $\bm{u}(k)$ most commonly comes from the control policy conditioned on the observation at time $k$, written as  $\bm{\pi}_{\bm{\theta}}\left(\bm{u}(k) \mid \bm{y}(k)\right)$ with $\bm{\theta}$ being the parameter for the control policy. Let $c\left(\bm{\tau}\right)$ be the cost defined over the trajectory $\bm{\tau}$.  The goal is to optimize the control paramater $\bm{\theta}$ to minimize the expected cost over  trajectories, written as:
\vspace{-0.1cm}
\begin{equation}\label{eq:Opt_obj}
J(\bm{\theta})=\mathbb{E}_{\bm{\tau} \sim p_{\bm{\pi}_{\bm{\theta}}}}c\left(\bm{\tau}\right) ,    
\end{equation}
where $p_{\bm{\pi}_{\bm{\theta}}}$ is the probability distribution of  trajectory  subject to the policy $\bm{\pi}_{\bm{\theta}}$.
% is given by
% \begin{equation}\label{eq:prob_traj0}
% p_{\bm{\pi}_{\bm{\theta}}}(\bm{\tau})=p\left(\bm{x}(0)\right) \prod_{k=0}^{K-1} p\left(\bm{y}(k) \mid \bm{x}(k)\right) \bm{\pi}_{\bm{\theta}}\left(\bm{u}(k) \mid \bm{y}(k)\right) p\left(\bm{x}(k+1) \mid \bm{x}(k), \bm{u}(k)\right) .    
% \end{equation}
The definition of $c\left(\bm{\tau}\right)$ varies for different problems. 
For linear control policy, quadratic costs are most commonly utilized to convert the optimization into classical linear quadratic problems~\citep{fazel2018global,zheng2021sample}. There are typically not closed-form solutions for~\eqref{eq:Opt_obj} for other cost functions, e.g., $c\left(\bm{\tau}\right)=\sum_{i=1}^n\max_{k=0,\cdots,K-1}|x_i(k)|$ for the frequency control problem in power systems~\citep{cui2022tps}. In this case, gradient-based methods can be utilized to update $\bm{\theta}$, but the lack of system parameters makes it difficult to compute the gradient $ \nabla_{\bm{\theta}} J(\bm{\theta})$.

% For example, $c\left(\bm{\tau}\right)=\sum_{k=0}^{K-1}\bm{x}(k)^\top Q\bm{x}(k)+\bm{u}(k)^\top R\bm{u}(k)$ for classical , $c\left(\bm{\tau}\right)=\sum_{i=1}^n\allowbreak\max_{k=0,\cdots,K-1}|x_i(k)|$ for the power system frequency control in~\citep{cui2022tps} and $c\left(\bm{\tau}\right)=\sum_{k=0}^{K-1}||\bm{x}(k)||_1+\gamma||\bm{u}(k)||_1$ for the battery control in. 

% \begin{subequations}\label{eq:Optimization}
% \begin{align}
% \min_{\bm{\bm{\theta}}} & \quad J(\bm{\theta}) \label{subeq:Optimization_obj}\\
% \mbox{s.t. } 
%  &\bm{x}(k+1)=A \bm{x}(k)+B \bm{u}(k)   \label{subeq:Optimization_V}\\
%  & \bm{u}(k) = \bm{\theta} \bm{x}(k)\\
% & \bm{u} \text{ is stabilizing}\label{subeq:Optimization_stability}
% \end{align}
% \end{subequations}

% The goal in a reinforcement learning problem is to learn a policy, which defines a distribution over actions conditioned on states, $\bm{\pi}_{\bm{\theta}}\left(\bm{u}(k) \mid \bm{x}(k)\right)$.   

\paragraph{Direct policy gradient methods.} 
Reinforcement learning (RL) is proposed to update $\bm{\theta}$ through gradient descent, where the gradient $ \nabla_{\bm{\theta}} J(\bm{\theta})$ is  approximated  from trajectories of the system under the control policy $\bm{\pi}_{\bm{\theta}}$. 
For example, the policy gradient methods in~\citep{Sutton2018RL} shows that the gradient $ \nabla_{\bm{\theta}} J(\bm{\theta})$  can be equivalently computed by
\vspace{-0.1cm}
\begin{equation}\label{eq: PG_grad}
    \nabla_{\bm{\theta}} J(\bm{\theta}) = \mathbb{E}_{\bm{\tau} \sim p_{\bm{\pi}_{\bm{\theta}}}(\bm{\tau})}\left[c(\bm{\tau})\sum_{k=0}^{K-1} \nabla_{\bm{\theta}} \log p_{\bm{\pi}_{\bm{\theta}}}\left(\bm{u}(k) \mid \bm{y}(k)\right)\right], 
\end{equation}
where $K$ is the length of the trajectory. The terms inside the expectation can be computed purely from observations and actions in the trajectory.

If $\bm{C}=\bm{I}_n$ in~\eqref{eq:sys_full}, the states are directly measured and the control law is called state-feedback control. Otherwise, the control law is called output-feedback control. 
Note that if $\bm{C}$ is not full column rank, then $(\bm{y}(t),\bm{u}(t))$ cannot uniquely determine $\bm{y}(t+1)$. Namely, the system is not a Markov decision process with respect to $(\bm{y}(t),\bm{u}(t))$ and it is difficult to use generic RL algorithms based on the quality function $Q((\bm{y}(t),\bm{u}(t))$~\citep{jin2020sample}.
For illustration purpose, this paper focus on the policy gradient algorithm~\eqref{eq: PG_grad}, and we consider both state and output feedback control. 
% Thus,  $\bm{\theta}$ is update through  the gradient $ \nabla_{\bm{\theta}} J(\bm{\theta})$  approximated  from trajectories of the system under the control policy $\bm{\pi}_{\bm{\theta}}$.
% Most RL methods are built on the Markov Decision process where the  observation-action pair  $(\bm{y}(k),\bm{u}(k))$ at the time step $k$ can uniquely determine the observation at the next time step (i.e., $\bm{y}(k+1)$).

% On this basis,  a quality function $Q_{\bm{\theta}}((\bm{y}(t),\bm{u}(t))$ is typically learned to approximate the cost function $J(\bm{\theta})$. The control policy $\bm{\pi}_{\bm{\theta}}$ is optimized subsequently using the learned quality function $Q_{\bm{\theta}}((\bm{y}(t),\bm{u}(t))$.  However, if $\bm{C}$ is not full column rank, then $(\bm{y}(t),\bm{u}(t))$ cannot uniquely determine $\bm{y}(t+1)$. Namely, the system is not a Markov decision process with respect to $(\bm{y}(t),\bm{u}(t))$. Correspondingly, the RL algorithms based on the quality function $Q((\bm{y}(t),\bm{u}(t))$ cannot be directly used. 
% For illustration purpose, this paper focus on policy gradient algorithm~\eqref{eq: PG_grad}.

\vspace{-0.1cm}
\subsection{Approximate gradients with sampled trajectories} 
% Throughout this paper, we  consider the linear feedback control law $\bm{u}(k)=\bm{\theta} \bm{y}(k)$, with $\bm{\theta}$ being the trainable parameter.
When online interactions with the system are limited, computing the expectation in~\eqref{eq: PG_grad} is not trivial even when the states are directly observed.
In training, the expectation in~\eqref{eq: PG_grad} is approximated by the sample average over a large number of trajectories $\bm{\tau}$ collected from the system under the control policy $\bm{\pi}_{\bm{\theta}}$. Since the distribution $p_{\bm{\pi}_{\bm{\theta}}}(\bm{\tau})$ depends on the parameter $\bm{\theta}$, a new set of trajectories need to be collected after each iteration of updating $\bm{\theta}$. Thus, the number of samples increases with the batchsize (i.e., the number of trajectories in each episode) and the number of training episodes.
% However, sampling on the true system may be expensive and implementing the control policy that has not been trained well may be dangerous.  

We seek to update the control policy using  historical trajectories and thus do not need to interact with the system during  training.  Two challenges arise: \textit{(i) Distribution Shift.}  If the control policy changes, the distribution of 
the historical trajectories would be different from the true distribution $p_{\bm{\pi}_{\bm{\theta}}}(\bm{\tau})$, potentially resulting in large errors when computing~\eqref{eq: PG_grad}. 
\textit{(ii) Exploration.} Most RL methods need to add (sometimes large) perturbations on actions to encourage exploration, but training with a fixed set of historical trajectories may limit exploration. 
% the explore of potential search region with lower cost.
% The large perturbations of actions typically do not exist in the past trajectories, thus training with past trajectories limits the exploration of policies.

% This is the well-known challenge of distribution shift problem in offline RL.

\paragraph{End-to-End Trajectory Generation.} We propose to overcome the challenges of distribution shift and the lack of exploration through generating trajectories from historical data. In this paper, we focus on learning linear feedback control law $\bm{u}(k)=\bm{\theta} \bm{y}(k)$, with $\bm{\theta}\in\real^{m\times q}$ being the matrix of trainable parameters. For exploration, 
the action during training follows the control policy with perturbations $\bm{w}(k)$ as additive noise, written as
\vspace{-0.1cm}
\begin{equation}\label{eq:control_policy}
\bm{\pi}_{\bm{\theta}}\left(\bm{u}(k) \mid \bm{y}(k)\right):=\left\{\bm{u}(k)=\bm{\theta} \bm{y}(k)+\bm{w}(k), \bm{w}(k)\sim\mathcal{D}\right\},
\end{equation}
where $\mathcal{D}$ is the prescribed distribution for the perturbations. The variance of $\mathcal{D}$  is typically initialized to be large and then shrink with the training episode to achieve  exploration v.s. exploitation.

We provide a simple end-to-end approach to generate input-output trajectories without system identification, allowing it to  extend to systems where the states are not directly observed (i.e., when $\bm{C}$ is not full column rank). The generated trajectories  are guaranteed to  have the same distribution as $p_{\bm{\pi}_{\bm{\theta}}}(\bm{\tau})$ 
% in~\eqref{eq:prob_traj0}
for all $\bm{\theta}$ and $\mathcal{D}$.  We first show the trajectory generation algorithm for state-feedback control in Section~\ref{sec:TrajGen_FullState}, then generalize the results to output-feedback control in Section~\ref{sec: TrajGen_PartialState}.

\section{Trajectory Generation for State-Feedback Control}
\label{sec:TrajGen_FullState}
% \input{Trajectory_Generation}
% In this section, we show the trajectory generation algorithm when all states are directly observed \lh{(are fully observed)}, namely $\bm{C}=\bm{I}_n$ and $\bm{y}(t)=\bm{x}(t)$. We will extend the results to the case when $\bm{C}\neq\bm{I}_n$ in the next section.
In this section, we show the conditions when the linear combination of historic trajectories spans all possible trajectories. On this basis, we derive the algorithm to generate trajectories corresponding to updated control policies with   perturbations on actions for explorations.
\subsection{Span of historic trajectories}

% represented as
% $$
% \begin{aligned}
% \mathcal{T}_{[0, T-1]} :=\left[\begin{array}{ccccc}
% \bm{B} & \mathbbold{0}_{n\times m}  & \mathbbold{0}_{n\times m} & \cdots & \mathbbold{0}_{n\times m} \\
% \bm{A} \bm{B} & \bm{B} & \mathbbold{0}_{n\times m} & \cdots & \mathbbold{0}_{n\times m} \\
% \vdots & \vdots & \vdots & \ddots & \vdots \\
% \bm{A}^{T-2} \bm{B} & \bm{A}^{T-3} \bm{B} & \bm{A}^{T-4} \bm{B} & \cdots & \mathbbold{0}_{n\times m}
% \end{array}\right] 
% \quad\mathcal{O}_{[0, T-1]} :=\left[\begin{array}{c}
% \bm{I}_{n} \\
% \bm{A} \\
% \vdots \\
% \bm{A}^{T-1}
% \end{array}\right].
% \end{aligned}
% $$

Let $\bm{u}_{d,[0, L-1]}$ and $\bm{x}_{d,[0, L-1]}$ be the a length-$L$ input and state from past trajectory, and let the corresponding Hankel matrix $\mathcal{H}\in\real^{(Tm+Tn)\times (L-T+1})$ defined as
\vspace{-0.1cm}
\begin{equation}\label{eq:Hankel_u_x}
    \begin{split}
    &\underbrace{ 
    \begin{bmatrix}
\bm{U}_{0, T, L-T+1}  \\
\bm{X}_{0, T, L-T+1}     
    \end{bmatrix}
}_{\mathcal{H}}\quad
:=
\left[\begin{matrix}
\bm{u}_d(0) & \bm{u}_d(1) & \cdots & \bm{u}_d(L-T) \\
% \bm{u}_d(1) & \bm{u}_d(2) & \cdots & \bm{u}_d(L-T+1) \\
\vdots & \vdots & \ddots & \vdots \\
\bm{u}_d(T-1) & \bm{u}_d(T) & \cdots & \bm{u}_d(L-1) \\
\hline \bm{x}_d(0) & \bm{x}_d(1) & \cdots & \bm{x}_d(L-T) \\
% \bm{x}_d(1) & \bm{x}_d(2) & \cdots & \bm{x}_d(L-T+1) \\
\vdots & \vdots & \ddots & \vdots \\
\bm{x}_d(T-1) & \bm{x}_d(T) & \cdots & \bm{x}_d(L-1)
\end{matrix}\right].    
    \end{split}
\end{equation}

By the state-space version of  Fundamental Lemma~\citep{de2019formulas} shown below, any linear combination of the columns of the Hankel matrix is a length-$T$ input-state trajectory of~\eqref{eq:sys_full}. The proof is provided in~\citep{de2019formulas} and we supplement it in Appendix~\ref{app: fundemental}  for completeness.
 
% \LH{use command ``bmatrix'' $\begin{bmatrix} a \\ b \end{bmatrix}$ would save you some space.}

\vspace{-0.1cm}
\begin{lemma}[ Fundamental Lemma]\label{lem:fundemental}
%  Consider system (1). Then, the following holds.
% 1) If $u_{d,[0, T-1]}$ is persistently exciting of order $n+t$, 
If $\text{rank}\left[\begin{smallmatrix}
\bm{U}_{0, T, L-T+1} \\
\bm{X}_{0, L-T+1}
\end{smallmatrix}\right]=n+Tm$,
then any length-$T$ input/state trajectory of system~\eqref{eq:sys_full} can be expressed as
$\left[\begin{smallmatrix}
\bm{u}_{[0, T-1]} \\
\bm{x}_{[0, T-1]}
\end{smallmatrix}\right]=\left[\frac{\bm{U}_{0, T, L-T+1}}{\bm{X}_{0, T, L-T+1}}\right] \bm{g}$,
where $\bm{g}\in \mathbb{R}^{L-T+1}$.

% Moreover, any linear combination of the columns of the Hankel matrix, that is
% $$
% \left[\begin{array}{c}
% \bm{U}_{0, T, L-T+1} \\
% \hline \bm{X}_{0, T, L-T+1}
% \end{array}\right] \bm{g}
% $$
% is a length-$T$ input/state trajectory of~\eqref{eq:sys_full}.
\end{lemma}

For the rank condition in Lemma~\ref{lem:fundemental} to hold, the minimum requirement on the length of the collected trajectory is $L-T+1=n+Tm$, namely, $L=(m+1)T-1+n$. 
When the rank condition holds, linear combination of the columns of the Hankel matrix is also a length-$T$ trajectory of the system. Thus, we generate a trajectory of length $T$ using
\vspace{-0.1cm}
\begin{equation}\label{eq:Hankel_g0}
\begin{bmatrix}
\Tilde{\bm{u}}(0);
\cdots;
\Tilde{\bm{u}}(T-1);
\Tilde{\bm{x}}(0);
\cdots;
\Tilde{\bm{x}}(T-1)
\end{bmatrix}
=   
   \underbrace{ 
    \begin{bmatrix}
\bm{U}_{0, T, L-T+1}  \\
\bm{X}_{0, T, L-T+1}     
    \end{bmatrix}
}_{\mathcal{H}}\bm{g}.
\end{equation}

% \begin{equation}\label{eq:Hankel_g0}
%  \left[\begin{array}{c}
% \Tilde{\bm{u}}(0) \\
% \vdots\\
% \Tilde{\bm{u}}(T-1)\\
% \Tilde{\bm{x}}(0)\\
% \vdots\\
% \Tilde{\bm{x}}(T-1)
% \end{array}\right]=   
%   \underbrace{ \left[\begin{array}{cccc}
% \bm{U}_{0, T, L-T+1}  \\
% \bm{X}_{0, T, L-T+1} 
% \end{array}\right]}_{\mathcal{H}}\bm{g}
% \end{equation}

For convenience, we adopt the notation $\mathcal{H}_u=\bm{U}_{0, T, L-T+1} $, $\mathcal{H}_x=\bm{X}_{0, T, L-T+1} $ in the following sections. To represent the rows of blocks starting from the time $k=0,\cdots,T-1$, we denote
\vspace{-0.1cm}
\begin{equation}\label{eq:H_u_k}
\mathcal{H}_x^k :=
[
\bm{x}_d(k) \; \bm{x}_d(k+1)\,  \cdots\,  \bm{x}_d(L-T+k)
],\,
\mathcal{H}_u^k :=
[
\bm{u}_d(k) \; \bm{u}_d(k+1) \, \cdots \, \bm{u}_d(L-T+k)
].  
\end{equation}

% \begin{equation*}
% \mathcal{H}_x^k :=
% \begin{bmatrix}
% \bm{x}_d(k) & \bm{x}_d(k+1) & \cdots & \bm{x}_d(L-T+k)
% \end{bmatrix}
% \mathcal{H}_u^k :=
% \begin{bmatrix}
% \bm{u}_d(k) & \bm{u}_d(k+1) & \cdots & \bm{u}_d(L-T+k)
% \end{bmatrix}.  
% \end{equation*}

% \begin{equation}\label{eq:H_u_k}
% \mathcal{H}_u^k :=
% \begin{bmatrix}
% \bm{u}_d(k) & \bm{u}_d(k+1) & \cdots & \bm{u}_d(L-T+k)
% \end{bmatrix}.    
% \end{equation}

\subsection{Trajectory generation}

For generic RL algorithms, a trajectory is sampled from the system that starts from an initial state $\bm{x}(0)$ and subsequently implements the control policy  $\bm{u}(k)=\bm{\theta} \bm{x}(k)+\bm{w}(k), \bm{w}(k)\sim\mathcal{D}$  for $k=0,\cdots,T-1$. Given $\bm{\theta}$, the probability density function of a trajectory is
\vspace{-0.1cm}
\begin{equation}\label{eq:pdf}
\begin{split}
   p_{\pi_{\bm{\theta}}}(\bm{\tau})= p\left(\bm{x}(0)\right) \prod_{k=0}^{T-1} &p\left(\bm{\theta} \bm{x}(k)+\bm{w}(k)|\bm{x}(k)\right) p\left(\bm{x}(k+1) \mid \bm{x}(k), \bm{\theta} \bm{x}(k)+\bm{w}(k)\right),
\end{split}
\end{equation}
which is uniquely determined by $\bm{x}(0)$ and the sequence of perturbations $\bm{w}(0), \bm{w}(1),\cdots, \bm{w}(T-1)$.

In the following, we generate trajectories for each updated $\bm{\theta}$ as if  they  truly come from the system starting from $\bm{x}(0)$ under perturbations on actions given by  $\bm{w}(0),\cdots, \bm{w}(T-1)$. 
 Importantly, we use fixed historic trajectories $\left\{\bm{u}_d, \bm{y}_d\right\}$ where the Hankel matrix $\mathcal{H}$ in~\eqref{eq:Hankel_u_x} satisfies $\text{rank}(\mathcal{H})=n+Tm$, and $\bm{u}_d$ can come from controllers different from the control policy in~\eqref{eq:control_policy}.

 The key is to use $\bm{x}(0)$ and $(\bm{w}(0),\allowbreak\cdots, \bm{w}(T-1))$ as extra constraints to find the $\bm{g}$ in~\eqref{eq:Hankel_g0} that generates the trajectory which follows the same distribution as~\eqref{eq:pdf}. 
From the control policy $\bm{u}(k)=\bm{\theta} \bm{x}(k)+\bm{w}(k)$, the trajectory subject to the perturbations  $\bm{w}(0), \bm{w}(1),\cdots, \bm{w}(T-1)$  should satisfy
\vspace{-0.1cm}
\begin{equation}\label{eq: costrains_w}
\begin{split}
 \left[\begin{array}{c}
\Tilde{\bm{u}}(0) \\
\vdots\\
\Tilde{\bm{u}}(T-1)\\
\end{array}\right]= 
\underbrace{\begin{bmatrix}
\bm{\theta}&  &  &  \\
 % &  \bm{\theta}&  &  \\
 &  &  \ddots&  \\
 &  &  & \bm{\theta} \\
\end{bmatrix}}_{\bm{I}_{T}\bigotimes\bm{\theta}}
 \left[\begin{array}{c}
\Tilde{\bm{x}}(0)\\
\vdots\\
\Tilde{\bm{x}}(T-1)
\end{array}\right] 
+\left[\begin{array}{c}
\bm{w}(0) \\
\vdots\\
\bm{w}(T-1)\\
\end{array}\right].    
\end{split}
\end{equation}

% \begin{equation}
%  \left[\begin{array}{c}
% \bm{u}(0) \\
% \vdots\\
% \bm{u}(T-1)\\
% \end{array}\right]= 
% \underbrace{\begin{bmatrix}
% \bm{\theta}&  &  &  \\
%  &  \bm{\theta}&  &  \\
%  &  &  \ddots&  \\
%  &  &  & \bm{\theta} \\
% \end{bmatrix}}_{\bm{I}_{T}\bigotimes\bm{\theta}}
%  \left[\begin{array}{c}
% \bm{x}(0)\\
% \vdots\\
% \bm{x}(T-1)
% \end{array}\right] 
% +\left[\begin{array}{c}
% \bm{w}(0) \\
% \vdots\\
% \bm{w}(T-1)\\
% \end{array}\right].
% \end{equation}

% Hence,the generated trajectory should also satisfy the system dynamics, which leads to
% \begin{equation}\label{eq: costrains_w}
% \begin{split}
% &\left[\begin{array}{c}
% \mathcal{H}_u
% \end{array}\right]\bm{g}\quad=\left(\bm{I}_{T}\otimes\bm{\theta}\right)
% \left[\begin{array}{cccc}
% \mathcal{H}_x
% \end{array}\right]\bm{g}
% +\left[\begin{array}{c}
% \bm{w}(0) \\
% \vdots\\
% \bm{w}(T-1)\\
% \end{array}\right].    
% \end{split}
% \end{equation}

Note that the generated initial state is given by $\Tilde{\bm{x}}(0)=\left[\begin{array}{cccc}
\mathcal{H}_x^0
\end{array}\right]\bm{g}$. Combining with~\eqref{eq: costrains_w} gives

% \begin{equation}
% \left[\begin{array}{cccc}
% \mathcal{H}_x^0
% \end{array}\right]:=\left[\begin{array}{cccc}
%  \bm{x}_d(0) & \bm{x}_d(1) & \cdots & \bm{x}_d(L-T) \\
% \end{array}\right],    
% \end{equation}

% Let
% \begin{equation}
%     R_{\bm{\theta}}:=\left[\begin{array}{cccc}
% \bm{U}_{0, T, L-T+1} 
% \end{array}\right]- \begin{bmatrix}
% \bm{\theta}&  &  &  \\
%  &  \bm{\theta}&  &  \\
%  &  &  \ddots&  \\
%  &  &  & \bm{\theta} \\
% \end{bmatrix}  
% \left[\begin{array}{cccc}
% \bm{X}_{0, T, L-T+1} 
% \end{array}\right]
% \end{equation}
\vspace{-0.1cm}

\begin{equation}\label{eq:traj_g0}
\underbrace{\begin{bmatrix}
\mathcal{H}_u-\left(\bm{I}_{T}\otimes\bm{\theta}\right) \mathcal{H}_x\\
\mathcal{H}_x^0
\end{bmatrix}}_{\bm{G}_{\bm{\theta}}}
\bm{g}=
\begin{bmatrix}
\bm{w}(0) \\
\vdots\\
\bm{w}(T-1)\\
\bm{x}(0)
\end{bmatrix}.
\end{equation}

Note that the matrix $\bm{G}_{\bm{\theta}}\in\real^{(Tm+n) \times (L-T+1)}$ is not a square matrix and its rank is determined by the length of historic trajectory $L$. When the trajectory is sufficiently long and  $\text{rank}(\mathcal{H})=n+Tm$, we have $L-T+1>Tm+n$ and thus there might be multiple $\bm{g} $ where~\eqref{eq:traj_g0} holds.
% Since all solution of~\eqref{eq:traj_g0} generates the unique trajectory, 
We use the  minimum-norm solution of~\eqref{eq:traj_g0} given by 
% \lh{(not sure if it is a good idea to mention noise in data, which also motivates the minimum-norm solution. I guess the reviewers will anyway ask you about how to deal with noisy data, but it is also okay to just let them ask.)}
\vspace{-0.1cm}
\begin{equation}\label{eq: g_star0}
    \bm{g}^*=\bm{G}_{\bm{\theta}}^\top\left(\bm{G}_{\bm{\theta}} \bm{G}_{\bm{\theta}}^\top\right)^{-1}\left[\begin{array}{c}
\bm{w}(0);
\cdots;
\bm{w}(T-1);
\bm{x}(0)
\end{array}\right].
\end{equation}

% \begin{equation}\label{eq: g_star0}
%     g^*=\bm{G}_{\bm{\theta}}^\top\left(\bm{G}_{\bm{\theta}} \bm{G}_{\bm{\theta}}^\top\right)^{-1}\left[\begin{array}{c}
% \bm{w}(0) \\
% \vdots\\
% \bm{w}(T-1)\\
% \bm{x}(0)
% \end{array}\right].
% \end{equation}

In the next Theorem, we prove  the existence and uniqueness of the trajectory generated by~\eqref{eq:traj_g0}. The goal is to show that given $\left(\bm{w}(0),\cdots ,\bm{w}(T-1), \bm{x}(0)\right)$, any $\bm{g} $  that satisfies~\eqref{eq:traj_g0}  will generate the same trajectory using $\mathcal{H}\bm{g}$. So it is suffice to choose the closed-form solution in~\eqref{eq: g_star0}.

\vspace{-0.1cm}
\begin{theorem}\label{thm: unque_trajectory0}
If $\text{rank}(\mathcal{H})=n+Tm$, there exists at least one solution $\bm{g}$ such that~\eqref{eq:traj_g0} holds. 
Given $\left(\bm{w}(0),\cdots ,\bm{w}(T-1), \bm{x}(0)\right)$ and any $\bm{g} $  that solves~\eqref{eq:traj_g0},  $\mathcal{H}\bm{g}^*$ generates the same unique trajectory under the control policy~\eqref{eq:control_policy} parameterized by $\bm{\theta}$.
% For any solution $g^*$ of~\eqref{eq:traj_g0}, the generated trajectory given by 
% $[\Tilde{\bm{u}}(0);
%         \cdots;
%         \Tilde{\bm{u}}(T-1);
%         \Tilde{\bm{x}}(0);
%         \cdots;
%         \Tilde{\bm{x}}(T-1)]
%         =\mathcal{H}\bm{g}^*$ 
% % $\Big[
% %         \Tilde{\bm{u}}(0)^\top 
% %         \cdots
% %         \Tilde{\bm{u}}(T-1)^\top
% %         \Tilde{\bm{x}}(0)^\top 
% %         \cdots
% %         \Tilde{\bm{x}}(T-1)^\top
% %         \Big]^\top=\mathcal{H}\bm{g}^*$ 
%         is the unique trajectory starting from $\bm{x}(0)$ under perturbations on actions given by  $\bm{w}(0),\cdots, \bm{w}(T-1)$.
\end{theorem}

Theorem~\ref{thm: unque_trajectory0} shows that $\mathcal{H}\bm{g}^*$  generates the unique trajectory that starts from an initial state $\bm{x}(0)$ and subsequently implements the control policy  $\bm{u}(k)=\bm{\theta} \bm{x}(k)+\bm{w}(k), \bm{w}(k)\sim\mathcal{D}$. Hence, if we fix $\bm{\theta}$ in~\eqref{eq: g_star0} and sample $(\bm{x}(0), \bm{w}(0), \bm{w}(1),\cdots, \bm{w}(T-1))$, then $\mathcal{H}\bm{g}^*$  will generate a batch of  trajectories following the distribution~\eqref{eq:pdf} corresponding to the parameter $\bm{\theta}$. After the update of $\bm{\theta}$ in each episode of training, we generate a new batch of trajectories by updating $\bm{G}_{\bm{\theta}}$ in~\eqref{eq:traj_g0} and sampling new $(\bm{x}(0), \bm{w}(0), \bm{w}(1),\cdots, \bm{w}(T-1))$. Thus, the generated trajectories adaptively follow the shifted distribution after updating $\bm{\theta}$. Importantly, it overcomes the negative perception in the field that there is no possibility to improve explorations beyond past trajectories~\citep{levine2020offline}. By sampling the noises $\bm{w}(k)\sim\mathcal{D}$, explorations can also be achieved through the generated trajectories. Hence, new trajectories are adaptively generated as if the system were being operated and explored under the updated control policies.

To prove Theorem~\ref{thm: unque_trajectory0}, we first show that the null space of $\bm{G}_{\bm{\theta}}$ is exactly the same as  that of the Hankel matrix $\mathcal{H}$. Then,  $\text{rank}(\mathcal{H})=n+Tm$ yields $\text{rank}(\bm{G}_{\bm{\theta}})=n+Tm$. The full row rank of $\bm{G}_{\bm{\theta}}$ implies that there exist at least one $\bm{g}$ where~\eqref{eq:traj_g0} holds. The uniqueness of trajectory generated by $\bm{g}$  follows from the fact that $\bm{G}_{\bm{\theta}}$ and $\mathcal{H}$ have the same null space. Details of the proof is given below.

% \begin{lemma}\label{lem:null_full}
% The null space $\mathcal{N}(\bm{G}_{\bm{\theta}})$ is  the same as $\mathcal{N}(\mathcal{H})$. Moreover, if $\text{rank}(\mathcal{H})=n+Tm$, then the matrix $\bm{G}_{\bm{\theta}}\in\real^{(Tm+n) \times (L-T+1)}$ is full row rank. Namely, $\text{rank}(\bm{G}_{\bm{\theta}})=n+Tm$. 
% \end{lemma}
\begin{proof}
We first prove that the null space $\mathcal{N}(\bm{G}_{\bm{\theta}})$ is  the same as $\mathcal{N}(\mathcal{H})$ from \textit{(i)}  and \textit{(ii)} :

\textit{(i)} For all $\bm{q}\in \mathcal{N}(\mathcal{H})$,  we have $[\mathcal{H}_x]\bm{q} =\mathbbold{0}_{Tn}$ and $[\mathcal{H}_u]\bm{q} =\mathbbold{0}_{Tm}$. Plugging in  $\bm{G}_{\bm{\theta}}$ in~\eqref{eq:traj_g0} yields  $\bm{G}_{\bm{\theta}}\bm{q} = \mathbbold{0}_{Tm+n} $. Namely, $\bm{q}\in\mathcal{N}(\bm{G}_{\bm{\theta}})$.

\textit{(ii)} For all  $\bm{v}\in \mathcal{N}(\bm{G}_{\bm{\theta}})$,  $\bm{G}_{\bm{\theta}}\bm{v}=\mathbbold{0}_{Tm+n}$ yields $\mathcal{H}_x^0\bm{v} =\mathbbold{0}_{n}$ 
 and 
 $\mathcal{H}_u^k \bm{v}=\hat{\bm{\theta}}\mathcal{H}_x^k \bm{v} \text{ for  } k=0,\cdots,T-1$. Thus, $\mathcal{H}_u^0 \bm{v}=\hat{\bm{\theta}}\mathcal{H}_x^0 \bm{v}=\mathbbold{0}_{m}$. From $\bm{x}(k+1) =\bm{A} \bm{x}(k)+\bm{B} \bm{u}(k)$, we have $\mathcal{H}_x^{k+1}=A \mathcal{H}_x^{k}+B \mathcal{H}_u^{k}.$ From $ \mathcal{H}_x^{0}=\mathbbold{0}_{n}$ and $ \mathcal{H}_u^{0}=\mathbbold{0}_{m}$, we
apply $\mathcal{H}_x^{k+1}=A \mathcal{H}_x^{k}+B \mathcal{H}_u^{k}$ and $\mathcal{H}_u^k \bm{v}=\hat{\bm{\theta}}\mathcal{H}_x^k \bm{v}$ alternately. This induces
$\mathcal{H}_x^k \bm{v} =\mathbbold{0}_{n}$ and $\mathcal{H}_u^k \bm{v} =\mathbbold{0}_{m}$ for $ k=0,\cdots,T-1$. Hence, $\mathcal{H}\bm{v} = \mathbbold{0}_{Tm+Tn} $. 
% Namely, $\bm{v}\in\mathcal{N}(\mathcal{H})$. 
% From \textit{(i)}  and \textit{(ii)} , $\mathcal{N}(\bm{G}_{\bm{\theta}})$ is  the same as $\mathcal{N}(\mathcal{H})$.
% \begin{equation}
% \left[\begin{array}{c}
% \mathcal{H}_u-\left(\bm{I}_{T}\otimes\bm{\theta}\right) \mathcal{H}_x\\
% \mathcal{H}_x^0
% \end{array}\right]\bm{v}=\mathbbold{0}_{Tm+n}
% \end{equation}
% This gives 
% \begin{equation}\label{eq:fullC_null}
% \begin{split}
% &\mathcal{H}_x^0\bm{v} =\mathbbold{0}_{n} 
% \quad and \quad 
%  \mathcal{H}_u^k \bm{v}=\hat{\bm{\theta}}\mathcal{H}_x^k \bm{v} \text{ for  } k=0,\cdots,T-1
% \end{split}    
% \end{equation} 

% When $C$ is full-rank square matrix, we have  $y(k+1)=Cx(k+1)=C(A x(k)+B u(k))=CA C^{-1}y(k)+CB u(k)$. Hence,

% \begin{equation}~\label{eq:induction_x_u}
%     \mathcal{H}_x^{k+1}=A \mathcal{H}_x^{k}+B \mathcal{H}_u^{k}.
% \end{equation}

% Plugging~\eqref{eq:fullC_null} in~\eqref{eq:induction_x_u} induces
% $\mathcal{H}_x^k \bm{v} =\mathbbold{0}_{n}$ and $\mathcal{H}_u^k \bm{v} =\mathbbold{0}_{m}$ for $ k=0,\cdots,T-1$. Hence, $\mathcal{H}\bm{v} = \mathbbold{0}_{Tm+Tn} $. Namely, $\bm{v}\in\mathcal{N}(\mathcal{H})$.

Next, we prove the existence of the solution in~\eqref{eq:traj_g0}. Note that $\mathcal{H}\in\real^{(Tm+Tn) \times (L-T+1)}$. If $\text{rank}(\mathcal{H})=n+Tm$, then  $\text{rank}(\mathcal{N}(\mathcal{H}))=(L-T+1)-(n+Tm)$.  
Since $\mathcal{N}(\bm{G}_{\bm{\theta}})$ is  the same as $\mathcal{N}(\mathcal{H})$, then $\text{rank}(\mathcal{N}(\bm{G}_{\bm{\theta}}))=(L-T+1)-(n+Tm)$.  It follows directly that $\text{rank}(\bm{G}_{\bm{\theta}})=(L-T+1)-\text{rank}(\mathcal{N}(\bm{G}_{\bm{\theta}}))=n+Tm$.  Note that the number of rows of  $\bm{G}_{\bm{\theta}}$ is $n+Tm$, then the full row-rank of $\text{rank}(\bm{G}_{\bm{\theta}})$ shows that there exists at least one solution such that~\eqref{eq:traj_g0} holds.

Lastly, we show the uniqueness of the generated trajectory. Suppose there exists $\bm{g}_1$ and $\bm{g}_2$, which are both solutions of~\eqref{eq:traj_g0} and $\mathcal{H}\bm{g}_1\neq \mathcal{H}\bm{g}_2$. Since $\bm{g}_1$ and $\bm{g}_2$ are both solution of~\eqref{eq:traj_g0}, then $\bm{G}_{\bm{\theta}}\bm{g}_1=\bm{G}_{\bm{\theta}}\bm{g}_1$ and thus $\left(\bm{g}_1-\bm{g}_2\right)\in\mathcal{N}(\bm{G}_{\bm{\theta}})$.
On the other hand, $\mathcal{H}\bm{g}_1\neq \mathcal{H}\bm{g}_2$ yields $\mathcal{H}\left(\bm{g}_1-\bm{g}_2\right)\neq 0$ and thus $\left(\bm{g}_1-\bm{g}_2\right)\notin\mathcal{N}(\mathcal{H})$. This contradicts that $\mathcal{N}(\bm{G}_{\bm{\theta}})$ is  the same as $\mathcal{N}(\mathcal{H})$. Hence, $\mathcal{H}\bm{g}_1= \mathcal{H}\bm{g}_2$, namely, the generated trajectories are identical.
% Therefore, $\mathcal{H}\bm{g}^*$ generates the unique trajectory for any solution $g^*$ of~\eqref{eq:traj_g0}.
\end{proof}

\vspace{-0.1cm}
\subsection{Algorithm}

By Theorem~\ref{thm: unque_trajectory0}, using historical data, given $\bm{x}(0)$ and perturbations on actions  $\bm{w}(0),\cdots, \bm{w}(T-1)$, 
a trajectory can be generated as if it comes from sampling the system with the current control policy.
The details of the algorithm is illustrated in Algorithm~\ref{alg: TrajGenFull}. 
We also use REINFORCE policy gradient~\citep{Sutton2018RL} in Appendix~\ref{app: PG_state_feedback} as an example to show how to use the trajectory generation algorithm in RL methods.
% We also use REINFORCE policy gradient~\citep{Sutton2018RL} in Algorithm~\ref{alg: PG} as an example to show how to integrate the trajectory generation algorithm in RL methods to train the controllers. 
The key benefit of Algorithm~\ref{alg: TrajGenFull}
% the trajectory generation algorithm 
is that it is adaptive to the updates of parameter $\bm{\theta}$ and  $\bm{w}(t)$ for exploration. In particular, $\bm{w}(t)$ can be sampled from any distribution $\mathcal{D}$, making it versatile for different applications.

\begin{algorithm}[h]
 \caption{Trajectory generation for state-feedback control }
 \label{alg: TrajGenFull}
 \begin{algorithmic}[1]
 \renewcommand{\algorithmicrequire}{\textbf{Require: }}
 \renewcommand{\algorithmicensure}{\textbf{Data collection:}}
%  \REQUIRE  The length $T$ of trajectory to be generated, the dimension of state $n$, the dimension of action $m$
 \ENSURE Collect historic measurement of the system and stack each $T$-length input-output trajectory  as Hankel matrix $\mathcal{H}$ shown in~\eqref{eq:Hankel_u_x} until  $\text{rank}(\mathcal{H})=n+Tm$ \\
  \renewcommand{\algorithmicensure}{\textbf{Data generation:}}
 \ENSURE 
 \textit{Input} :Hankel matrix $\mathcal{H}$, weights $\bm{\theta}$ and the distribution $\mathcal{D}$ for the control policy,  the batchsize $Q$ (number of the generated  trajectories), the distribution $\mathcal{S}_x $ of the initial states\footnotemark \\
     \SetKwFunction{FMain}{TrajectoryGen}
    \SetKwProg{Fn}{Function}{:}{}
    \Fn{\FMain{$\mathcal{H},\bm{\theta},\mathcal{D},Q, \mathcal{S}_x  $}}{
          Plug in $\bm{\theta}$ to compute $\bm{G}_{\bm{\theta}}=\begin{bmatrix}
\mathcal{H}_u-\left(\bm{I}_{T}\otimes\bm{\theta}\right) \mathcal{H}_x;
\mathcal{H}_x^0
\end{bmatrix}$

        \For{$i = 1$ to $Q$}
        {Sample $\bm{x}_i(0)$ from $\mathcal{S}_x  $. Sample $\left\{\bm{w}_i(0),\cdots,\bm{w}_i(T-1)\right\}$ from  $\mathcal{D}$.\\
        Compute the  coefficient $\bm{g}_i^*=\bm{G}_{\bm{\theta}}^\top\left(\bm{G}_{\bm{\theta}} \bm{G}_{\bm{\theta}}^\top\right)^{-1}\left[
        \bm{w}_i(0);
        \cdots;
        \bm{w}_i(T-1);
        \bm{x}_i(0)
        \right]$.
        \\
        Generate the $i$-th trajectory    $\bm{\tau}_i:=\left[
        \Tilde{\bm{u}}_i(0);
        \cdots;
        \Tilde{\bm{u}}_i(T-1);
        \Tilde{\bm{x}}_i(0);
        \cdots;
        \Tilde{\bm{x}}_i(T-1);
        \right]=\mathcal{H}\bm{g}_i^*$.

        }
        \textbf{return} $ \left[\bm{\tau}_1,\cdots,\bm{\tau}_Q \right]$ 
        }
    % \textbf{End Function}
 \end{algorithmic} 
 \end{algorithm}
\footnotetext{The set of historic initial states in past trajectories can be used to estimate $\mathcal{S}_x $.}

\section{Trajectory Generation for Output-Feedback Control}
\label{sec: TrajGen_PartialState}

% For the systems where the states cannot be directly observed ( $\bm{C}\neq\bm{I}_n$), output feedback control is adopted. If $\bm{C}$ is a full column rank, then $\bm{x}(t)$ can be uniquely recovered from $\bm{y}(t)$ using the relation  $\bm{x}(t)=(\bm{C}^\top\bm{C})^{-1}\bm{C}^\top\bm{y}(k)$. Correspondingly, $\bm{y}(k+1)=\bm{C}(\bm{C}^\top\bm{C})^{-1}\bm{C}^\top\bm{y}(k)+\bm{C}\bm{B}\bm{u}(k)$. Then the algorithm in Section~\ref{sec:TrajGen_FullState} directly works by replacing $\bm{x}$ with $\bm{y}$.
% On the other hand, if $\bm{C}$ is not full column rank, then $\bm{y}(k+1)$ cannot be determined by $(\bm{y}(k),\bm{u}(k))$, namely,
% the system is not a Markov decision process with respect to $(\bm{y}(k),\bm{u}(k))$ at one time step. In the following part of this section, we focus on the case when $\bm{C}$ is not full column rank. We will show that the system is a Markov decision process by defining an extended state using input-output
%  trajectory. On this basis, the trajectory generation algorithm also works by defining on the extended state. 

% As discussed in Subsection~\ref{subsec: problem}, $\bm{y}(k+1)$ cannot be determined by $(\bm{y}(k),\bm{u}(k))$ if $\bm{C}$ is not full column rank for the output-feedback control. Thus, the algorithm derived in the previous section cannot be directly used by replacing $\bm{x}$ using  $\bm{y}$. 
In this section, we show how to obtain a Markov decision process by defining an extended state using input-output
 trajectory. 
 % Then the high-level idea of trajectory generation is identical to Section~\ref{sec:TrajGen_FullState}.
 The key difference to Section~\ref{sec:TrajGen_FullState} is that $\bm{G}_{\bm{\theta}}$ may not be full row rank even when the rank condition on the Hankel matrix holds. Thus, the coefficients for generated trajectories and associated proofs are more nuanced. 

\subsection{Extended states for constructing Markov decision process}

% We show that the system is a Markov decision process in terms of trajectory of observations.

Let $
\mathcal{O}_{[0, \ell]}(\bm{A}, \bm{C}):=\operatorname{col}\left(\bm{C}, \bm{C} \bm{A}, \ldots, \bm{C} \bm{A}^{\ell-1}\right)
$ be the extended observability matrix.
The lag of the system~\eqref{eq:sys_full} is defined by the smallest integer $\ell\in \mathbb{Z}_{\geq 0}$ such that the observability matrix $\mathcal{O}_{[0, \ell]}(\bm{A}, \bm{C})$ has rank $n$, i.e., the state can be reconstructed from $\ell$ measurements~\citep{huang2021robust}.

Let $T_0\geq\ell$ be the length of a trajectory. 
Define the extended states as
\vspace{-0.1cm}
\begin{equation}
\mathcal{X}(k-1) :=\left[\begin{array}{c}
\bm{y}(k-T_0);
\cdots ;
\bm{y}(k-1);
\bm{u}(k-T_0);
\cdots ;
\bm{u}(k-2)
\end{array}\right].        
\end{equation}

% \begin{equation}
% \mathcal{X}(k-1) :=\left[\begin{array}{c}
% \bm{y}(k-T_0)\\
% \vdots \\
% \bm{y}(k-1)\\
% \bm{u}(k-T_0)\\
% \vdots \\
% \bm{u}(k-2)\\
% \end{array}\right]        
% \end{equation}

Then extending the system transition from time step 0 to $T_0$ gives 
\vspace{-0.1cm}
\begin{equation}\label{eq:sys_extended}
\mathcal{X}(k)=\Tilde{\bm{A}}\mathcal{X}(k-1)+\Tilde{\bm{B}}\bm{u}(k-1)
\text{ for } k\geq T_0, k\in\mathbb{Z},    
\end{equation}
which is a Markov decision process in terms of the extended states. Detailed proof and the definition of system transition matrix $(\Tilde{\bm{A}},\Tilde{\bm{B}})$ is given in Appendix~\ref{app: transition_partial}.
For the output-feedback control law in~\eqref{eq:control_policy},
% $\bm{u}(k)=\bm{\theta} \bm{y}(k)+\bm{w}(k)$
we have  $ p\left(\bm{u}(k)|\mathcal{X}(k)\right)=p\left(\bm{u}(k)|\bm{y}(k)\right)   
$ and it is straightforward to show that policy gradient algorithm using~\eqref{eq: PG_grad} still works. The proof is given in Appendix~\ref{app:Policy Gradient}.

% \section{Sample efficient algorithm}
% \label{sec: sample}
% The policy gradient algorithm needs more samples to converge, so we want to derive sample-efficient algorithm.

% \subsection{Fundamental Lemma}
By defining the the Hankel matrix as $\mathcal{H}:=\begin{bmatrix}
\bm{U}_{0, T, L-T+1} \\
\bm{Y}_{0, T, L-T+1}
\end{bmatrix}\in\real^{(Tm+Tq)\times(L-T+1)}$, the following fundamental Lemma in terms of input-output trajectories holds. 
\vspace{-0.1cm}
\begin{lemma}[Fundamental Lemma~\citet{willems2005note, markovsky2022identifiability} ]\label{lem:fundemental_partial}

If $\text{rank}\left[\begin{smallmatrix}
\bm{U}_{0, T, L-T+1} \\
\bm{Y}_{0, L-T+1}
\end{smallmatrix}\right]=n+Tm$,
then any length-$T$ input/output trajectory of system~\eqref{eq:sys_full} can be expressed as
$
\left[\begin{smallmatrix}
\bm{u}_{[0, T-1]} \\
\bm{y}_{[0, T-1]}
\end{smallmatrix}\right]
=\left[\frac{\bm{U}_{0, T, L-T+1}}{\bm{Y}_{0, T, L-T+1}}\right] \bm{g}
$
where $\bm{g} \in \mathbb{R}^{L-T+1}$.

% Moreover, any linear combination of the columns of the Hankel matrix, that is
% $$
% \left[\begin{array}{c}
% \bm{U}_{0, T, L-T+1} \\
% \hline \bm{Y}_{0, T, L-T+1}
% \end{array}\right]\bm{g}
% $$
% is a length-$T$ input/output trajectory of~\eqref{eq:sys_full}.
\end{lemma}

When the rank condition in Lemma~\ref{lem:fundemental_partial} holds, linear combination of the columns of the Hankel matrix is an input/output trajectory of the system. We then generate trajectory of length $T$ using
\vspace{-0.1cm}
\begin{equation}\label{eq:Hankel_g1}
 \left[\begin{array}{c}
\Tilde{\bm{u}}(0);
\cdots;
\Tilde{\bm{u}}(T-1);
\Tilde{\bm{y}}(0);
\cdots;
\Tilde{\bm{y}}(T-1)
\end{array}\right]=   
%     \underbrace{ 
%   \begin{bmatrix}
% \bm{U}_{0, T, L-T+1}  \\
% \bm{Y}_{0, T, L-T+1} 
% \end{bmatrix}}_
\mathcal{H}\bm{g}. 
\end{equation}
% _{\mathcal{H}}\bm{g}

% \begin{equation}\label{eq:Hankel_g1}
%  \left[\begin{array}{c}
% \Tilde{\bm{u}}(0) \\
% \vdots\\
% \Tilde{\bm{u}}(T-1)\\
% \Tilde{\bm{y}}(0)\\
% \vdots\\
% \Tilde{\bm{y}}(T-1)
% \end{array}\right]=   
%   \underbrace{ \left[\begin{array}{cccc}
% \bm{U}_{0, T, L-T+1}  \\
% \bm{Y}_{0, T, L-T+1} 
% \end{array}\right]}_{\mathcal{H}}\bm{g}
% \end{equation}

For convenience, we adopt the notation $\mathcal{H}_u=U_{0, T, L-T+1} $, $\mathcal{H}_u=Y_{0, T, L-T+1} $ in the following sections. To represent the lines of blocks starting from the time $k=0,\cdots,T-1$, $\mathcal{H}_u^k$ the same as in~\eqref{eq:H_u_k} and $\mathcal{H}_y^k :=[
 y_d(k) \, \cdots\,y_d(L-T+k) ]. $ We also define the stacked blocks starting from $k_1, \cdots, k_2$ as $
\mathcal{H}_y^{k_1 : k_2} :=
\begin{bmatrix}
 {H}_y^{k_1 };
  {H}_y^{k_1 +1};
 \cdots;
   {H}_y^{k_2}
\end{bmatrix} $ and $
\mathcal{H}_u^{k_1 : k_2} :=
\begin{bmatrix}
 {H}_u^{k_1 };
  {H}_u^{k_1 +1};
 \cdots;
   {H}_u^{k_2}
\end{bmatrix}$.

% \begin{equation}
% \mathcal{H}_y^{k_1 : k_2} :=
% \begin{bmatrix}
%  {H}_y^{k_1 } \\
%   {H}_y^{k_1 +1} \\
%  \vdots\\
%   {H}_y^{k_2}
% \end{bmatrix}
% \quad
% \mathcal{H}_u^{k_1 : k_2} :=
% \begin{bmatrix}
%  {H}_u^{k_1 } \\
%   {H}_u^{k_1 +1} \\
%  \vdots\\
%   {H}_u^{k_2}
% \end{bmatrix}.    
% \end{equation}

\subsection{Trajectory generation}\label{subsec: TrajGenPartial}
Using the transition dynamics~\eqref{eq:sys_extended}, the probability density function of a length-$T$ trajectory is
\vspace{-0.1cm}
\begin{equation*}\label{eq:prob_traj_partial}
    p_{\pi_{\bm{\theta}}}(\bm{\tau})= p\left(\mathcal{X}\left(T_0-1\right)\right) \prod_{k=T_0-1}^{T-1} p\left(\bm{\theta} \bm{y}(k)+\bm{w}(k)|\mathcal{X}(k)\right) p\left(\mathcal{X}(k+1)|\mathcal{X}(k), \bm{\theta} \bm{y}(k)+\bm{w}(k)\right),
\end{equation*}
which is uniquely determined by $\mathcal{X}\left(T_0-1\right)$ and the sequences  $\bm{w}(T_0-1), \cdots, \bm{w}(T-1)$.

In analogy with the derivation in~\eqref{eq: costrains_w}, we aim to generate the trajectory starting from $\mathcal{X}\left(T_0-1\right)$ under perturbations on actions given by  $\bm{w}(T_0-1),\cdots, \bm{w}(T-1)$. From the control policy $\bm{u}(k)=\bm{\theta} \bm{y}(k)+\bm{w}(k)$, the trajectory subject to the perturbations $\bm{w}(T_0), \cdots, \bm{w}(T-1)$ should satisfy 
\vspace{-0.1cm}
\begin{equation}\label{eq: costrains_w_partial}
\begin{split}
&\begin{bmatrix}
\mathcal{H}_u^{T_0 -1: T-1}
\end{bmatrix}\bm{g}=\left(\bm{I}_{T-T_0}\otimes\bm{\theta}\right)
\begin{bmatrix}
\mathcal{H}_y^{T_0 -1: T-1}
\end{bmatrix}\bm{g}
+\begin{bmatrix}
\bm{w}(T_0-1) \\
\vdots\\
\bm{w}(T-1)\\
\end{bmatrix}.    
\end{split}
\end{equation}

% Let
% \begin{equation}
%     R_\theta:=\left[\begin{array}{cccc}
% U_{0, T, L-T+1} 
% \end{array}\right]- \begin{bmatrix}
%  \theta&  &  &  \\
%  &   \theta&  &  \\
%  &  &  \ddots&  \\
%  &  &  &  \theta \\
% \end{bmatrix}  
% \left[\begin{array}{cccc}
% Y_{0, T, L-T+1} 
% \end{array}\right]
% \end{equation}

% For a partially observable system, $\text{rank}(Y_{0, 1, L-T+1})<n$  and thus  $\text{rank}(G_\theta)<n+Tm$. This makes  $\mathcal{N}(G_\theta)$ larger than $\mathcal{N}(\mathcal{H})$.   Hence, different solution of~\eqref{eq:traj_g} create different trajectory and the associated cost. This makes sense because there are several $x(0)$  that leads to the same $y(0)$ for a partially observable  system. 

% Hence, if we would like the trajectory generating algorithm to work in partially observable system, we need to  make use of the trajectory of $y(t)$ from $x(0)$. The idea is to first collect an initial trajectory of $y(t)$ and $u(t)$ to inherently determine $x(0)$. This constitutes a extra constraint to derive an unique $g$. Let $T_0$ be the 
% length of initial trajectory, where the rank condition is given by~\eqref{eq:y_u_traj_transition}. We need $T_0$ long enough to make the observability matrix  $\mathcal{O}_{[0, T_0-1]}$ full column rank.

Note that the generated extended initial state is given by $\Tilde{\mathcal{X}}\left(T_0-1\right):=\begin{bmatrix}
\mathcal{H}_y^{0 : T_0-1} ;
\mathcal{H}_u^{0 : T_0-2}
\end{bmatrix}\bm{g}$. Together with the constraints in~\eqref{eq: costrains_w_partial} gives 
\vspace{-0.1cm}
% \begin{equation}\label{eq:traj_g_partial}
% \begin{split}
% \left[\begin{array}{c}
% \mathcal{H}_u^{T_0 -1: T-1} -\left(\bm{I}_{T-T_0}\otimes\bm{\theta}\right)\mathcal{H}_y^{T_0 -1: T-1} \\
% \hline
% \mathcal{H}_y^{0 : T_0-1}\\
% \mathcal{H}_u^{0 : T_0-2}\\
% \end{array}\right] 
% =
% \left[\begin{array}{c}
% \bm{w}(T_0) \\
% \vdots\\
% \bm{w}(T-1)\\
% \hline
% \mathcal{X}(T_0-1)
% % y^{\text{init}}(0)
% % \vdots\\
% % y^{\text{init}}(T_0-1)\\
% % u^{\text{init}}(0)
% % \vdots\\
% % u^{\text{init}}(T_0-1)\\
% \end{array}\right]
% \mbox{ or } {G_\theta}\bm{g}={\bm{R}} 
% \end{split}
% \end{equation}
% \vspace{-0.1cm}
\begin{equation}\label{eq:traj_g_partial}
\begin{split}
\underbrace{\left[\begin{array}{c}
\mathcal{H}_u^{T_0 -1: T-1} -\left(\bm{I}_{T-T_0}\otimes\bm{\theta}\right)\mathcal{H}_y^{T_0 -1: T-1} \\
\hline
\mathcal{H}_y^{0 : T_0-1}\\
\mathcal{H}_u^{0 : T_0-2}\\
\end{array}\right]}_{G_\theta}\bm{g}
=
\underbrace{\left[\begin{array}{c}
\bm{w}(T_0-1) \\
\vdots\\
\bm{w}(T-1)\\
\hline
\mathcal{X}(T_0-1)
% y^{\text{init}}(0)
% \vdots\\
% y^{\text{init}}(T_0-1)\\
% u^{\text{init}}(0)
% \vdots\\
% u^{\text{init}}(T_0-1)\\
\end{array}\right]}_{\bm{R}}    
\end{split}
\end{equation}
\vspace{-0.6cm}

% We show in the next Lemma that the matrix $\bm{R}$ in the right side of equation~\eqref{eq:traj_g_partial}
% $\bm{G}_{\bm{\theta}}\in\real^{(Tm+T_0 d) \times (L-T+1)}$ in is not full row-rank. 

Note that the matrix $\bm{G}_{\bm{\theta}}\in\real^{(Tm+T_0 d) \times (L-T+1)}$ is not a square matrix and there might be multiple solutions to~\eqref{eq:traj_g_partial}. Moreover,  $\bm{G}_{\bm{\theta}}$ may not be full row rank and thus $\left(\bm{G}_{\bm{\theta}} \bm{G}_{\bm{\theta}}^\top\right)$ may not be invertible. Here, we compute the eigenvalue decomposition of $\left(\bm{G}_{\bm{\theta}} \bm{G}_{\bm{\theta}}^\top\right)$. Let $s$ be the number of nonzero eigenvalue of  $\left(\bm{G}_{\bm{\theta}} \bm{G}_{\bm{\theta}}^\top\right)$. Let $\lambda_i$ be the $i$-th non-zero eigenvalue and $\bm{p}_i$ be the associated eigenvector of  $\left(\bm{G}_{\bm{\theta}} \bm{G}_{\bm{\theta}}^\top\right)$. Denote $\bm{P}_{\bm{\theta}}:=\begin{bmatrix}
\bm{p}_1 & \bm{p}_2 &\cdots& \bm{p}_s
\end{bmatrix}$ and $\bm{\Lambda}=\text{diag}(\lambda_1,\lambda_2,\cdots,\lambda_s)$. Then clearly 
$\left(\bm{G}_{\bm{\theta}} \bm{G}_{\bm{\theta}}^\top\right)=\bm{P}_{\bm{\theta}}\bm{\Lambda}\bm{P}_{\bm{\theta}}^\top$ and we compute the solution of~\eqref{eq:traj_g_partial} given by
\vspace{-0.1cm}
\begin{equation}\label{eq: g_star_partial}
    \bm{g}^*=\bm{G}_{\bm{\theta}}^\top\bm{P}_{\bm{\theta}}\bm{\Lambda}^{-1}\bm{P}_{\bm{\theta}}^\top\left[\begin{array}{c}
\bm{w}(T_0-1);
\cdots;
\bm{w}(T-1);
\mathcal{X}(T_0-1)
\end{array}\right].
\end{equation}
\vspace{-0.6cm}

Next, we prove  the existence and uniqueness of the trajectory generated by~\eqref{eq:traj_g_partial}. The goal is to show that given $\left(\bm{w}(T_0-1),\cdots ,\bm{w}(T-1), \mathcal{X}(T_0-1)\right)$ , any
$\bm{g} $ that satisfies~\eqref{eq:traj_g_partial} will generate the same trajectory using $\mathcal{H}\bm{g}$. So it is suffice to choose the closed-form solution in~\eqref{eq: g_star_partial}. 
\begin{theorem}\label{thm: unque_trajectory0_partial}
If $\text{rank}(\mathcal{H})=n+Tm$, there exists at least one solution $\bm{g}^*$ such that~\eqref{eq:traj_g_partial} holds. 
Given $\left(\bm{w}(T_0-1),\cdots ,\bm{w}(T-1), \mathcal{X}(T_0-1)\right)$ and any $\bm{g} $  that solves~\eqref{eq:traj_g_partial},  $\mathcal{H}\bm{g}^*$ generates the same unique trajectory under the control policy~\eqref{eq:control_policy} parameterized by $\bm{\theta}$.
% For any solution $\bm{g}^*$ of~\eqref{eq:traj_g_partial}, the generated trajectory given by $[
%         \Tilde{\bm{u}}(T_0);
%         \cdots;
%         \Tilde{\bm{u}}(T-1);
%         \Tilde{\bm{y}}(T_0);
%         \cdots;
%         \Tilde{\bm{y}}(T-1)
%       ]=\mathcal{H}\bm{g}^*$ is the unique trajectory starting from $\mathcal{X}(T_0-1)$ under subsequent perturbations on actions given by  $\bm{w}(T_0),\cdots, \bm{w}(T-1)$.
\end{theorem}
The proof of Theorem~\ref{thm: unque_trajectory0_partial} is not as straightforward as Theorem~\ref{thm: unque_trajectory0}, because $\bm{G}_{\bm{\theta}}$ and  $\bm{R}$ in~\eqref{eq:traj_g_partial} may not be full row-rank. The detailed proof is given in Appendix~\ref{app:thm_partial} and we sketch the proof as follows.
We use the mapping from $\bm{x}(0)$ to $\mathcal{X}(T_0-1)$ to show that the rank of  $\bm{R}$ in~\eqref{eq:traj_g_partial} is at most $n+Tm$. Leveraging the relation in every $T_0$ blocks derived from~\eqref{eq:sys_extended}, we
show in Lemma~\ref{lem:null_partial} that $\text{rank}(\bm{G}_{\bm{\theta}})=n+Tm$ if $\text{rank}(\mathcal{H})=n+Tm$. The existence of a solution in~\eqref{eq:traj_g_partial} is therefore guaranteed by the same row-rank of the two sides. The uniqueness of the trajectory generated by  $\mathcal{H}\bm{g}^*$ is proved by showing that the Null space $\mathcal{N}(\bm{G}_{\bm{\theta}})$ is  the same as $\mathcal{N}(\mathcal{H})$.

We can generate a trajectory  $\mathcal{H}\bm{g}^*$ by randomly sampling $\mathcal{X}(T_0-1)$ and  $\bm{w}(t)\in\mathcal{D}$ for $t=T_0-1,\cdots,T-1$.  For the cost~\eqref{eq:Opt_obj} calculated on the trajectory of the length $K$, we setup $T=K+T_0-1$, and using the generated trajectory $\mathcal{H}\bm{g}^*$ from $T_0-1$ to $T-1$ to train the controller. The detailed  algorithm can be found in Appendix~\ref{app:algorithm}.

\vspace{-0.5cm}
\section{Experiment}
\label{sec: experiment}
We end the paper with case studies for the control of the batch reactor system in~\citep{de2019formulas} and the power distribution system in~\citep{cui2022decentralized}. Both state-feedback and output-feedback control are studied in these systems. Detailed simulation setting and results are provided in  Appendix~\ref{app:exp_reactor} and Appendix~\ref{app:exp_voltage}. 
Code is available at 
\url{https://github.com/Wenqi-Cui/Trajectory-Generation}. Major simulation results are summarized below. 
\paragraph{Experiment Setup.} 
 We use REINFORCE policy gradient algorithm~\citep{Sutton2018RL} to train a linear feedback controller, with the goal to minimize  the cost of trajectories with length $K$. 
Let $E$ be the number of episode in training and $Q$ be the batch size of trajectories collected for each episode, respectively. Standard policy gradient algorithms needs  $Q\times K \times E$ samples. 
 We compare the performance of the REINFORCE policy gradient algorithm using the generated trajectories (labeled as PG-TrajectoryGen) and the same algorithm using trajectories sampled by interacting with the system (labeled as PG-Sample-Q for the batchsize Q). For testing, we randomly sample 800 initial states and compare the cost on trajectories starting from these states. 
 
\paragraph{State-feedback control in a batch reactor system.}
We use the model of a batch reactor system in~\citep{de2019formulas} , where $\bm{x}(t)\in\real^4, \bm{u}(t)\in\real^2$. The time horizon is $K=30$. Since the state is observed, we setup $T=K=30$ and collect historic trajectory of length $L = (m + 1)T -1 + n=93$ such that $\text{rank}(\mathcal{H})=n+Tm$. With a sampling period of 0.1s, the data collection takes 9.3s. We generate trajectories using Algorithm~\ref{alg: TrajGenFull} with a batch size of 1200 in each  episode. 
% , we can generate any batch of trajectories without any new steps of sampled data from the real system. 
 Figure~\ref{fig:Loss_AB_full_main}(a) and Figure~\ref{fig:Loss_AB_full_main}(b) shows the training and the testing loss, respectively.
 PG-TrajectoryGen attains the same performance as  PG-Sample-1200, since Theorem~\ref{thm: unque_trajectory0} guarantees that we generate trajectories  as if it were truly sampled on the system. Moreover,  Figure~\ref{fig:Loss_AB_full_main}(a)-(c) shows that the loss attained by   PG-Sample-Q reduces with larger Q, but it comes at the expense of increased number of samples from the system. By contrast, the number of samples in PG-TrajectoryGen purely comes
from the fixed historic trajectory of the length $L=93$, which is much smaller than PG-Sample-10 that needs $10\times30\times400=120000$ online samples. 
% where the length of samples is 120000 (equals to $10\times30\times400$).

\vspace{-0.1cm}

\begin{figure}[H]
\centering
\subfigure[Training Loss]{\includegraphics[width=1.8in]{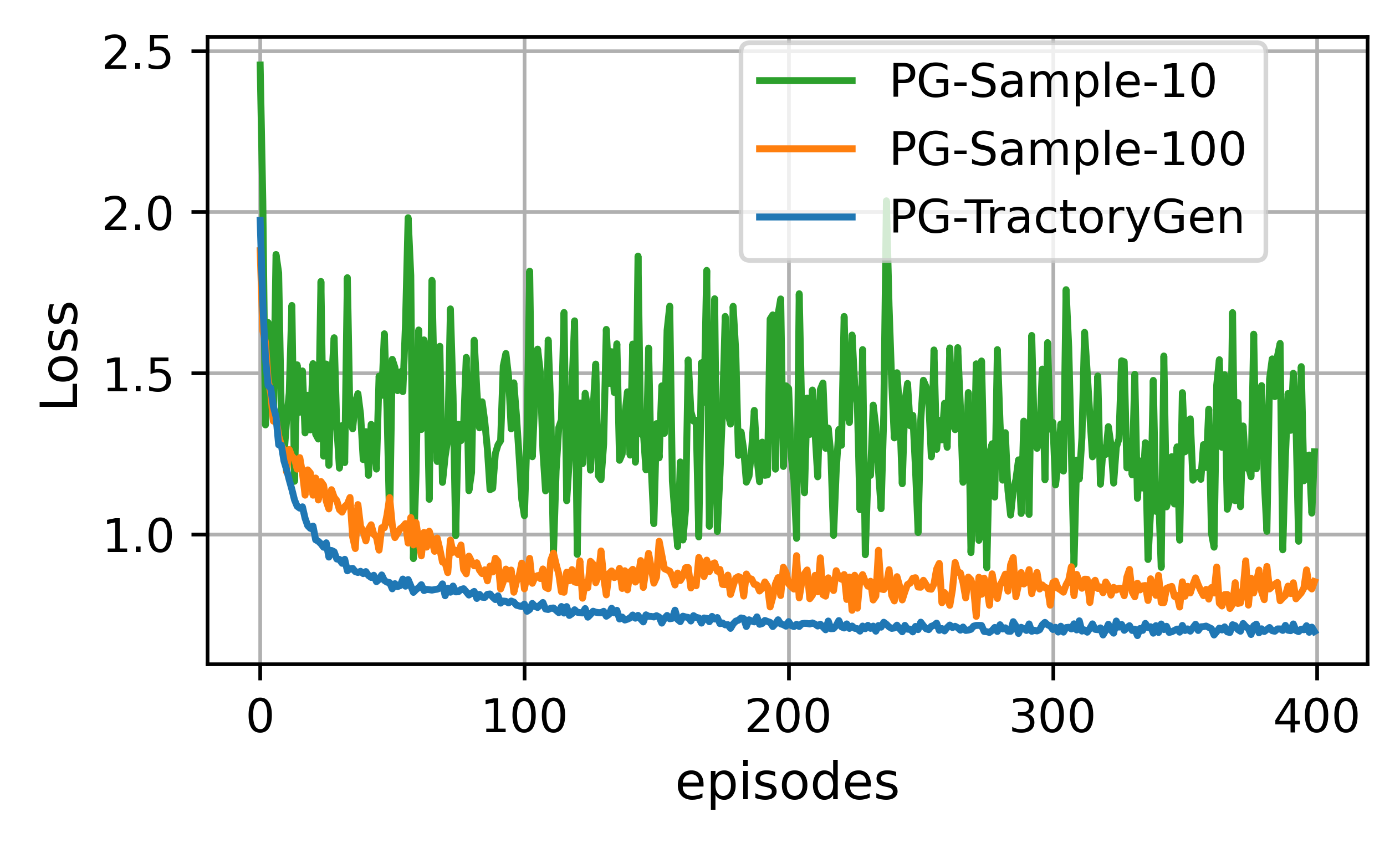}}
% \subfigure[Testing-FullObs]{\includegraphics[width=1.4in]{figure/Loss_ABfull_test.png}}
\subfigure[Testing Loss]{\includegraphics[width=1.8in]{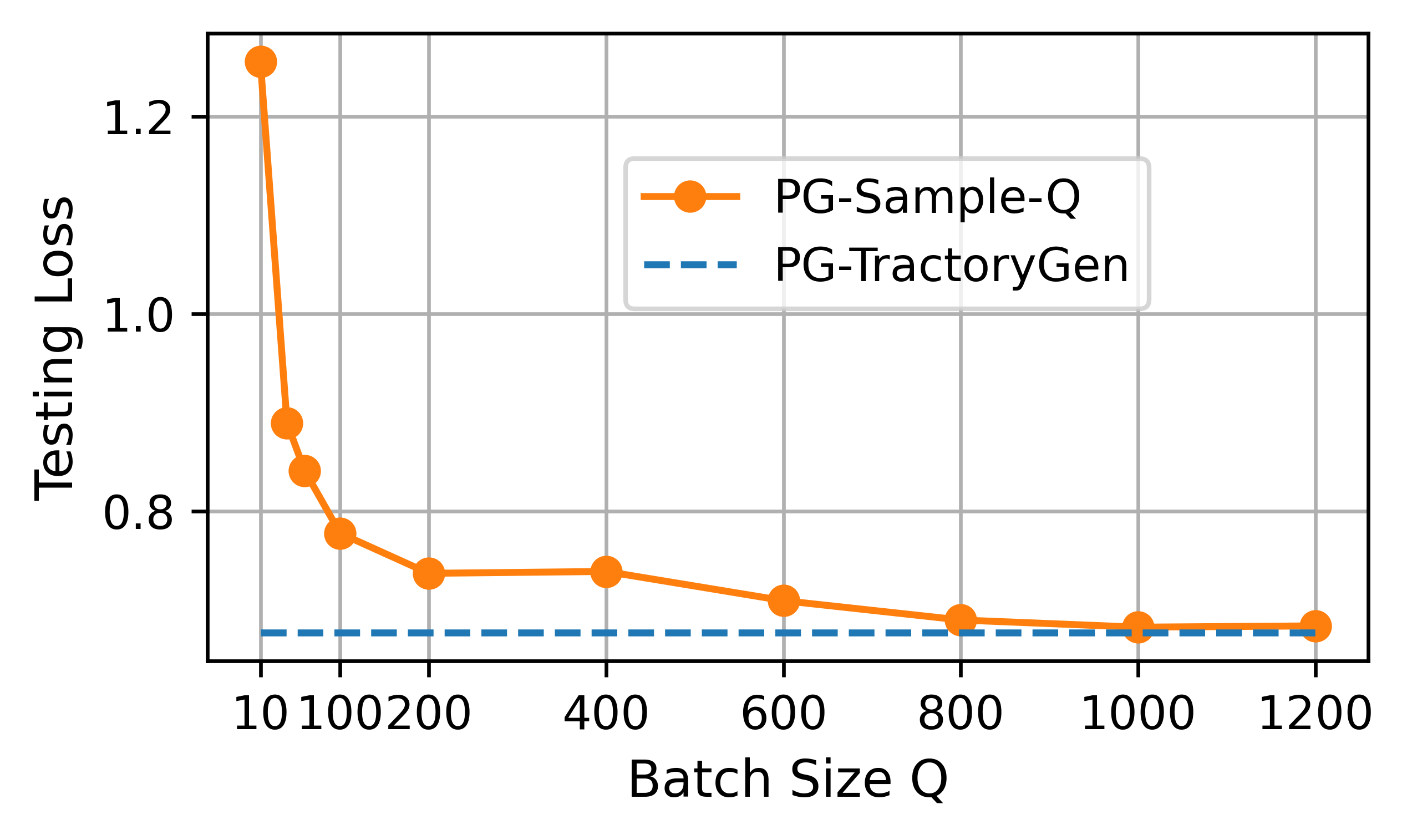}}
\subfigure[Number of samples]{\includegraphics[width=1.8in]{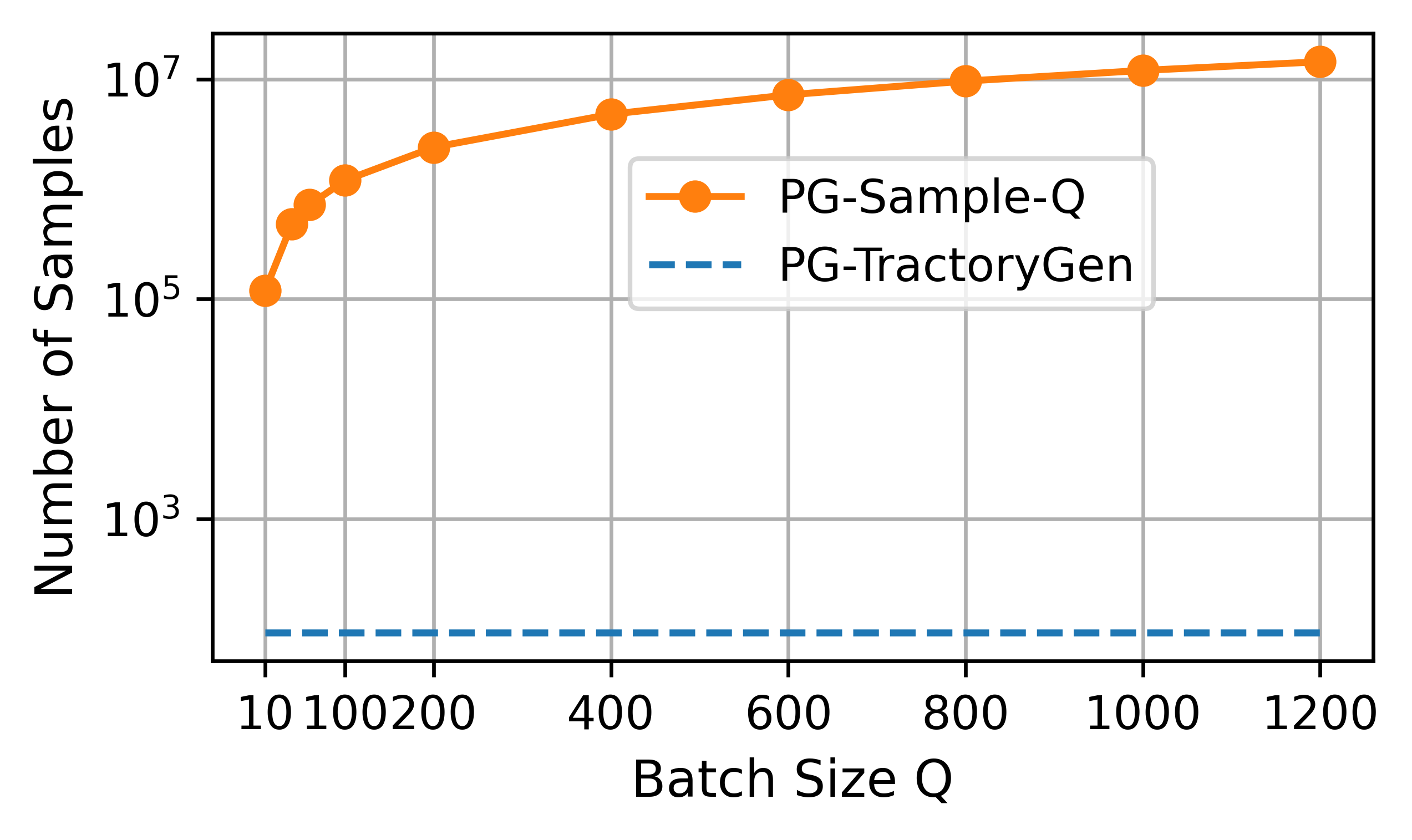}}
\vspace*{-3mm}
\caption{Performance of learning state-feedback controllers in the batch reactor system. PG-TrajectoryGen achieves the same loss as PG-Sample-1200 with much less samples.}\label{fig:Loss_AB_full_main}
\end{figure}

\vspace{-0.8cm}
\paragraph{Output-feedback control in a power distribution network.}
We further conduct experiments on the voltage control problem in IEEE 33bus test feeder~\citep{baran1989network, cui2022decentralized}, where $\bm{x}(t)\in\real^{32}$. We assume only 20 buses are measured and controlled, so $\bm{y}(t),\bm{u}(t)\in\real^{20}$. The observability matrix $\mathcal{O}_{[0, T_0]}(\bm{A}, \bm{C})$ becomes full column rank when $T_0=3$. The time horizon of trajectory is $K=20$. According to the trajectory generation algorithm developed in Subsection~\ref{subsec: TrajGenPartial}, we setup $T=K+T_0-1=22$ and collect historic trajectory of length $L = (m + 1)T -1 + n=493$. With a sampling period of 1s, data collection takes 493s. Figure~\ref{fig:Loss_voltage_partial_main}(a)-(c) compares the training loss, testing loss and the number of samples, respectively.  PG-TrajectoryGen achieves the same training and testing loss as PG-Sample-1000 with much smaller number of samples.

\vspace{-0.1cm}
\begin{figure}[H]
\centering
\subfigure[Training Loss]{\includegraphics[width=1.8in]{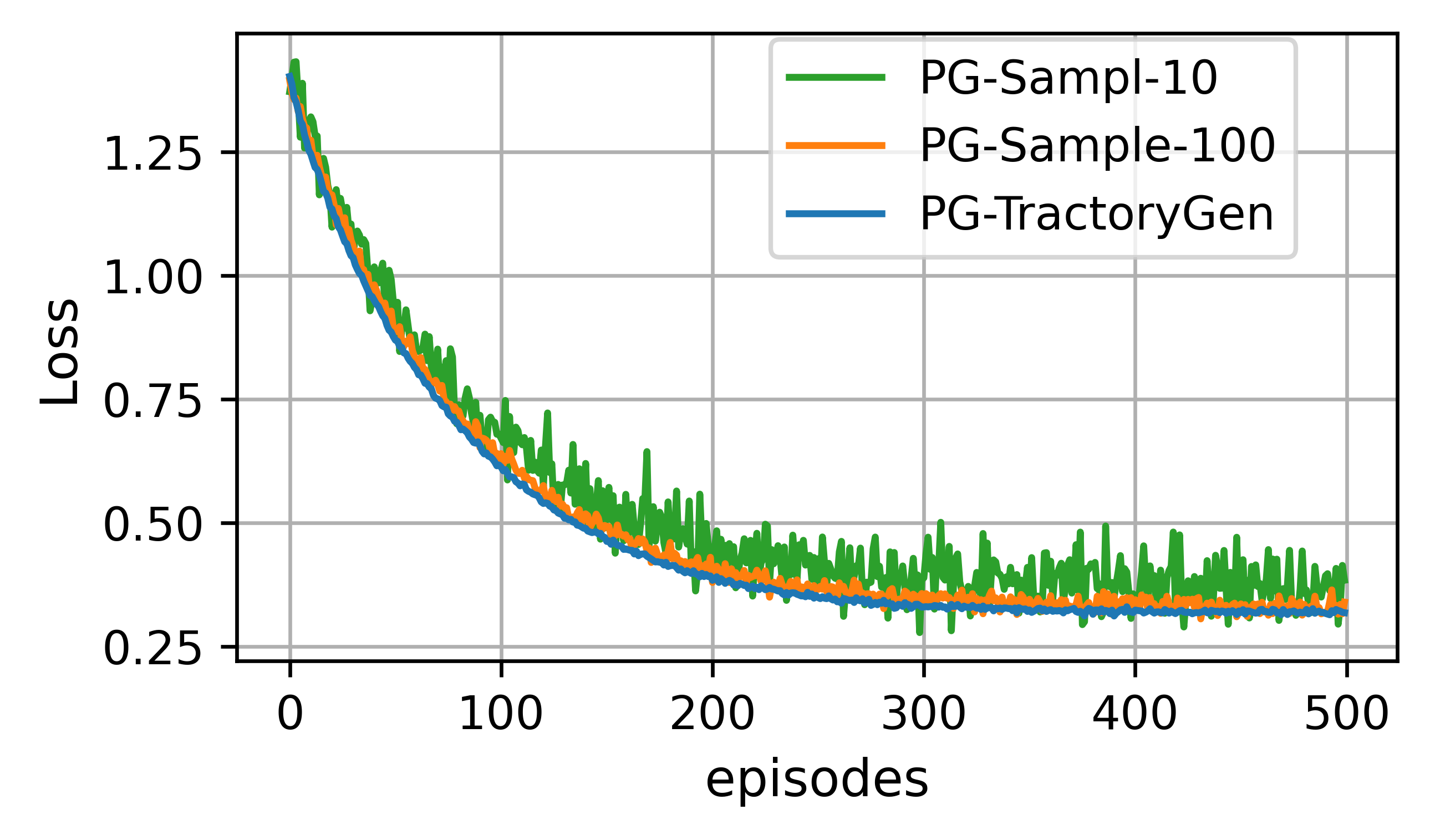}}
% \subfigure[Testing-FullObs]{\includegraphics[width=1.4in]{figure/Loss_ABfull_test.png}}
\subfigure[Testing Loss]{\includegraphics[width=1.8in]{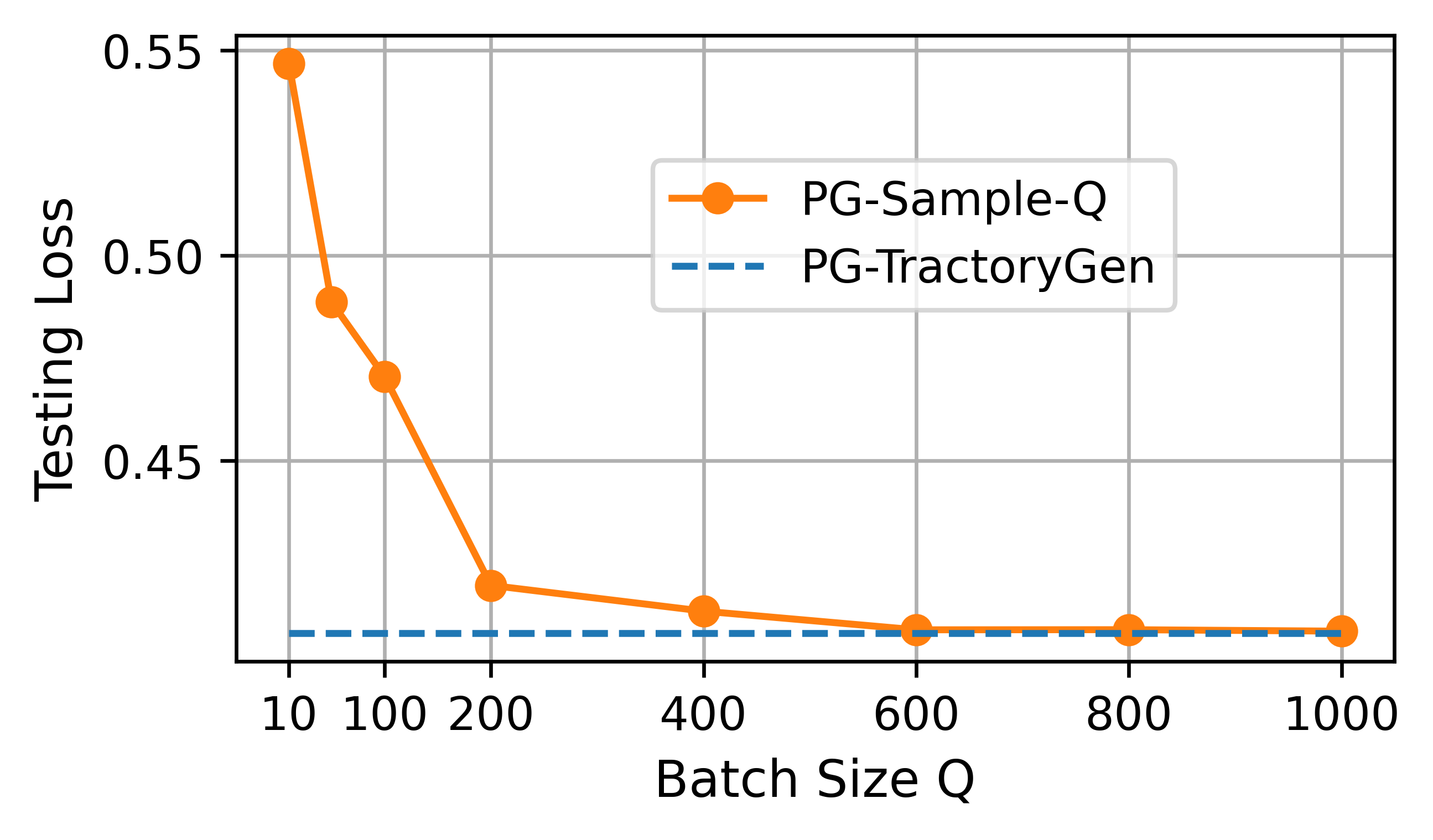}}
\subfigure[Number of samples]{\includegraphics[width=1.8in]{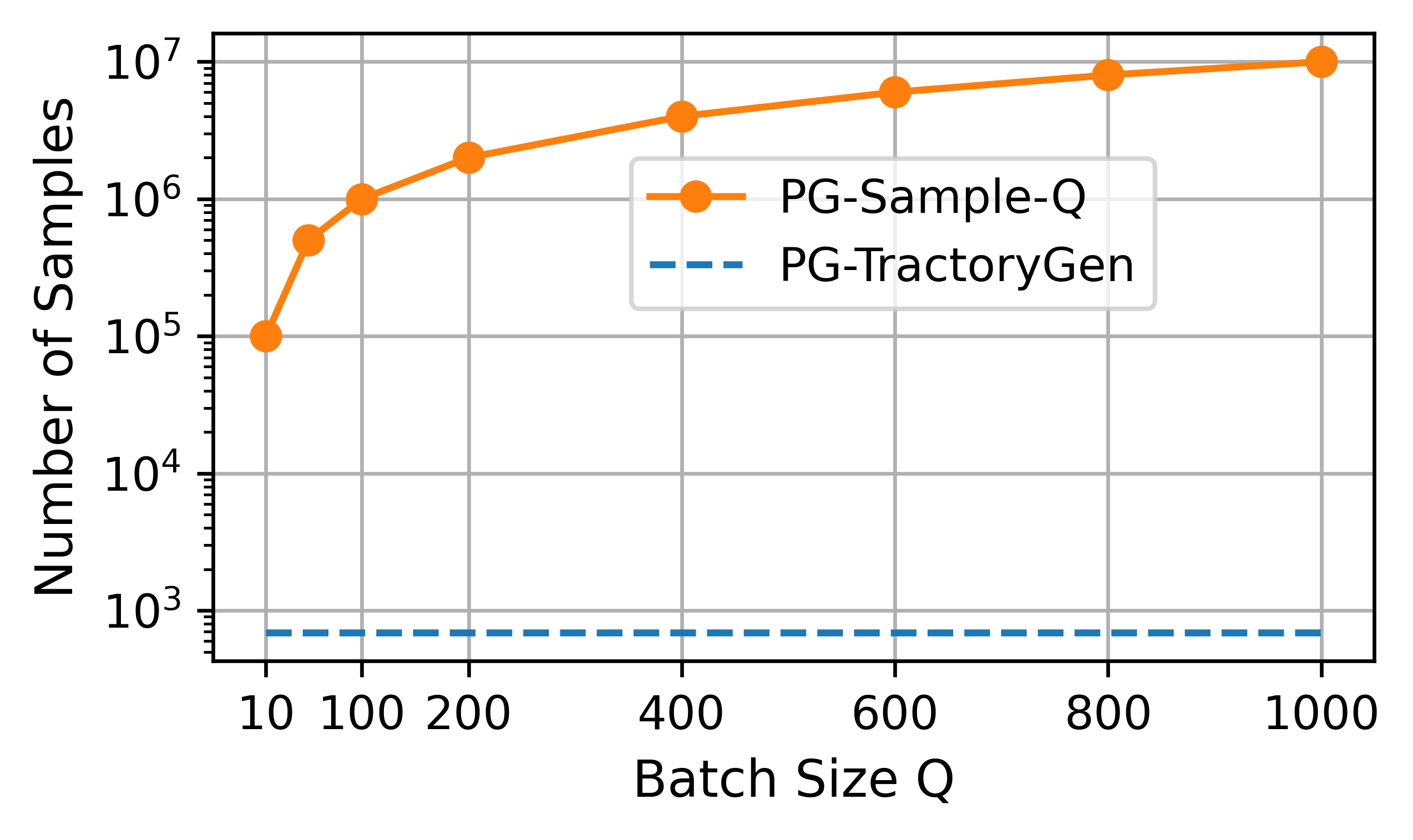}}
\vspace*{-3mm}
\caption{Performance of learning output-feedback controllers in power distribution network. PG-TrajectoryGen achieves the same loss as PG-Sample-1000 with much less samples.}\label{fig:Loss_voltage_partial_main}
\end{figure}

\vspace{-1.2cm}
\section{Conclusions}
This paper proposes a trajectory generation algorithm for learning linear feedback controllers in linear systems.
We prove that the algorithm generates trajectories with the exact distribution as if they are sampled by interacting with the real system using the updated control policy. In particular, the algorithm extends to systems where the states are not directly observed.  This is done by equivalently defining system transition dynamics using input-output trajectories.  Experiments show that the proposed method significantly reduces the number of sampled data needed for RL algorithms. 
% Acknowledgments---Will not appear in anonymized version
% \acks{We thank a bunch of people.}

\bibliography{ref}

\begin{thebibliography}{22}
\providecommand{\natexlab}[1]{#1}
\providecommand{\url}[1]{\texttt{#1}}
\expandafter\ifx\csname urlstyle\endcsname\relax
  \providecommand{\doi}[1]{doi: #1}\else
  \providecommand{\doi}{doi: \begingroup \urlstyle{rm}\Url}\fi

\bibitem[Baran and Wu(1989)]{baran1989network}
Mesut~E Baran and Felix~F Wu.
\newblock Network reconfiguration in distribution systems for loss reduction
  and load balancing.
\newblock \emph{IEEE Power Engineering Review}, 9\penalty0 (4):\penalty0
  101--102, 1989.

\bibitem[Chen et~al.(2021)Chen, Donti, Baker, Kolter, and
  Berg{\'e}s]{chen2021enforcing}
Bingqing Chen, Priya~L Donti, Kyri Baker, J~Zico Kolter, and Mario Berg{\'e}s.
\newblock Enforcing policy feasibility constraints through differentiable
  projection for energy optimization.
\newblock In \emph{Proceedings of the Twelfth ACM International Conference on
  Future Energy Systems}, pages 199--210, 2021.

\bibitem[Cui et~al.(2022{\natexlab{a}})Cui, Jiang, and Zhang]{cui2022tps}
Wenqi Cui, Yan Jiang, and Baosen Zhang.
\newblock Reinforcement learning for optimal primary frequency control: A
  {L}yapunov approach.
\newblock \emph{IEEE Transactions on Power Systems}, 2022{\natexlab{a}}.

\bibitem[Cui et~al.(2022{\natexlab{b}})Cui, Li, and
  Zhang]{cui2022decentralized}
Wenqi Cui, Jiayi Li, and Baosen Zhang.
\newblock Decentralized safe reinforcement learning for inverter-based voltage
  control.
\newblock \emph{Electric Power Systems Research}, 211:\penalty0 108609,
  2022{\natexlab{b}}.

\bibitem[De~Persis and Tesi(2019)]{de2019formulas}
Claudio De~Persis and Pietro Tesi.
\newblock Formulas for data-driven control: Stabilization, optimality, and
  robustness.
\newblock \emph{IEEE Transactions on Automatic Control}, 65\penalty0
  (3):\penalty0 909--924, 2019.

\bibitem[Fazel et~al.(2018)Fazel, Ge, Kakade, and Mesbahi]{fazel2018global}
Maryam Fazel, Rong Ge, Sham Kakade, and Mehran Mesbahi.
\newblock Global convergence of policy gradient methods for the linear
  quadratic regulator.
\newblock In \emph{International Conference on Machine Learning}, pages
  1467--1476. PMLR, 2018.

\bibitem[Fujimoto and Gu(2021)]{fujimoto2021minimalist}
Scott Fujimoto and Shixiang~Shane Gu.
\newblock A minimalist approach to offline reinforcement learning.
\newblock \emph{Advances in neural information processing systems},
  34:\penalty0 20132--20145, 2021.

\bibitem[Ghasemipour et~al.(2021)Ghasemipour, Schuurmans, and
  Gu]{ghasemipour2021emaq}
Seyed Kamyar~Seyed Ghasemipour, Dale Schuurmans, and Shixiang~Shane Gu.
\newblock Emaq: Expected-max q-learning operator for simple yet effective
  offline and online rl.
\newblock In \emph{International Conference on Machine Learning}, pages
  3682--3691. PMLR, 2021.

\bibitem[Gulcehre et~al.(2020)Gulcehre, Wang, Novikov, Paine, G{\'o}mez, Zolna,
  Agarwal, Merel, Mankowitz, Paduraru, et~al.]{gulcehre2020rl}
Caglar Gulcehre, Ziyu Wang, Alexander Novikov, Thomas Paine, Sergio G{\'o}mez,
  Konrad Zolna, Rishabh Agarwal, Josh~S Merel, Daniel~J Mankowitz, Cosmin
  Paduraru, et~al.
\newblock Rl unplugged: A suite of benchmarks for offline reinforcement
  learning.
\newblock \emph{Advances in Neural Information Processing Systems},
  33:\penalty0 7248--7259, 2020.

\bibitem[Hespanha(2009)]{hespanha2009linear}
Joao~P Hespanha.
\newblock \emph{Linear systems theory}.
\newblock Princeton university press, 2009.

\bibitem[Hu et~al.(2022)Hu, Zhang, Li, Mesbahi, Fazel, and
  Ba{\c{s}}ar]{hu2022towards}
Bin Hu, Kaiqing Zhang, Na~Li, Mehran Mesbahi, Maryam Fazel, and Tamer
  Ba{\c{s}}ar.
\newblock Towards a theoretical foundation of policy optimization for learning
  control policies.
\newblock \emph{arXiv preprint arXiv:2210.04810}, 2022.

\bibitem[Huang et~al.(2021)Huang, Zhen, Lygeros, and
  D{\"o}rfler]{huang2021robust}
Linbin Huang, Jianzhe Zhen, John Lygeros, and Florian D{\"o}rfler.
\newblock Robust data-enabled predictive control: Tractable formulations and
  performance guarantees.
\newblock \emph{arXiv preprint arXiv:2105.07199}, 2021.

\bibitem[Jin et~al.(2020)Jin, Kakade, Krishnamurthy, and Liu]{jin2020sample}
Chi Jin, Sham Kakade, Akshay Krishnamurthy, and Qinghua Liu.
\newblock Sample-efficient reinforcement learning of undercomplete pomdps.
\newblock \emph{Advances in Neural Information Processing Systems},
  33:\penalty0 18530--18539, 2020.

\bibitem[Jin et~al.(2021)Jin, Yang, and Wang]{jin2021pessimism}
Ying Jin, Zhuoran Yang, and Zhaoran Wang.
\newblock Is pessimism provably efficient for offline rl?
\newblock In \emph{International Conference on Machine Learning}, pages
  5084--5096. PMLR, 2021.

\bibitem[Levine et~al.(2020)Levine, Kumar, Tucker, and Fu]{levine2020offline}
Sergey Levine, Aviral Kumar, George Tucker, and Justin Fu.
\newblock Offline reinforcement learning: Tutorial, review, and perspectives on
  open problems.
\newblock \emph{arXiv preprint arXiv:2005.01643}, 2020.

\bibitem[Markovsky and D{\"o}rfler(2022)]{markovsky2022identifiability}
Ivan Markovsky and Florian D{\"o}rfler.
\newblock Identifiability in the behavioral setting.
\newblock \emph{IEEE Transactions on Automatic Control}, 2022.

\bibitem[Ostrovski et~al.(2021)Ostrovski, Castro, and
  Dabney]{ostrovski2021difficulty}
Georg Ostrovski, Pablo~Samuel Castro, and Will Dabney.
\newblock The difficulty of passive learning in deep reinforcement learning.
\newblock \emph{Advances in Neural Information Processing Systems},
  34:\penalty0 23283--23295, 2021.

\bibitem[Sutton and Barto(2018)]{Sutton2018RL}
Richard~S. Sutton and Andrew~G. Barto.
\newblock \emph{Reinforcement Learning: An Introduction}.
\newblock The MIT Press, 2nd edition, 2018.

\bibitem[Tang et~al.(2021)Tang, Zheng, and Li]{tang2021analysis}
Yujie Tang, Yang Zheng, and Na~Li.
\newblock Analysis of the optimization landscape of linear quadratic gaussian
  (lqg) control.
\newblock In \emph{Learning for Dynamics and Control}, pages 599--610. PMLR,
  2021.

\bibitem[Willems et~al.(2005)Willems, Rapisarda, Markovsky, and
  De~Moor]{willems2005note}
Jan~C Willems, Paolo Rapisarda, Ivan Markovsky, and Bart~LM De~Moor.
\newblock A note on persistency of excitation.
\newblock \emph{Systems \& Control Letters}, 54\penalty0 (4):\penalty0
  325--329, 2005.

\bibitem[Zheng et~al.(2021)Zheng, Furieri, Kamgarpour, and Li]{zheng2021sample}
Yang Zheng, Luca Furieri, Maryam Kamgarpour, and Na~Li.
\newblock Sample complexity of linear quadratic gaussian (lqg) control for
  output feedback systems.
\newblock In \emph{Learning for dynamics and control}, pages 559--570. PMLR,
  2021.

\bibitem[Zhu and Liu(2015)]{zhu2015fast}
Hao Zhu and Hao~Jan Liu.
\newblock Fast local voltage control under limited reactive power: Optimality
  and stability analysis.
\newblock \emph{IEEE Transactions on Power Systems}, 31\penalty0 (5):\penalty0
  3794--3803, 2015.

\end{thebibliography}
\newpage
\section{Appendix}
\label{sec: appendix}

\subsection{Fundemental Lemma}\label{app: fundemental}

For state feedback control,  $\bm{y}(t)=\bm{x}(t)$ ($\bm{C}=\bm{I}_n$) and the system transition dynamics is reduced to $\bm{x}(k+1) =\bm{A} \bm{x}(k)+\bm{B} \bm{u}(k)$.  Expanding the dynamics for $k=0, \cdots, T$ gives the input-state response  over $[0, T-1]$ as 
% $
% \left[\begin{smallmatrix}
% \bm{u}_{[0, T-1]} \\
% \bm{x}_{[0, T-1]}
% \end{smallmatrix}\right]=\left[\begin{smallmatrix}
% \bm{I}_{T m} & \mathbbold{0}_{T m \times n} \\
% \mathcal{T}_{[0,T-1]} & \mathcal{O}_{[0,T-1]}
% \end{smallmatrix}\right]
% \left[\begin{smallmatrix}
% \bm{u}_{[0, T-1]} \\
% \bm{x}(0)
% \end{smallmatrix}\right],
% $
\begin{equation}\label{eq: stack_transition}
\begin{bmatrix}
\bm{u}_{[0, T-1]} \\
\bm{x}_{[0, T-1]}
\end{bmatrix}=\begin{bmatrix}
\bm{I}_{T m} & \mathbbold{0}_{T m \times n} \\
\mathcal{T}_{[0,T-1]} & \mathcal{O}_{[0,T-1]}
\end{bmatrix}
\begin{bmatrix}
\bm{u}_{[0, T-1]} \\
\bm{x}(0)
\end{bmatrix},
\end{equation}
where $\mathcal{T}_{[0, T-1]}$  and $\mathcal{O}_{[0, T-1]}$ are the Toeplitz and observability matrices of order $T$ represented as~\citep{de2019formulas}
$$
\begin{aligned}
\mathcal{T}_{[0, T-1]} :=\left[\begin{array}{ccccc}
\bm{B} & \mathbbold{0}_{n\times m}  & \mathbbold{0}_{n\times m} & \cdots & \mathbbold{0}_{n\times m} \\
\bm{A} \bm{B} & \bm{B} & \mathbbold{0}_{n\times m} & \cdots & \mathbbold{0}_{n\times m} \\
\vdots & \vdots & \vdots & \ddots & \vdots \\
\bm{A}^{T-2} \bm{B} & \bm{A}^{T-3} \bm{B} & \bm{A}^{T-4} \bm{B} & \cdots & \mathbbold{0}_{n\times m}
\end{array}\right] 
\quad\mathcal{O}_{[0, T-1]} :=\left[\begin{array}{c}
\bm{I}_{n} \\
\bm{A} \\
\vdots \\
\bm{A}^{T-1}
\end{array}\right].
\end{aligned}
$$

% The representation of $\mathcal{T}_{[0, T-1]}$  and $\mathcal{O}_{[0, T-1]}$ are supplemented in Appendix~\ref{app: matrics} for completeness. 
On this basis, the Hankel Matrix $\mathcal{H}$ can be represented as
\begin{equation}\label{eq:Hankel_u_x(k)ransition}
\begin{bmatrix}
\bm{U}_{0, T, L-T+1} \\
\bm{X}_{0, T, L-T+1}
\end{bmatrix}
=\begin{bmatrix}
I_{T m} & \mathbbold{0}_{T m \times n} \\
\mathcal{T}_{[0,T-1]} & \mathcal{O}_{[0,T-1]}
\end{bmatrix}
\begin{bmatrix}
\bm{U}_{0, T, L-T+1} \\
X_{0, L-T+1}
\end{bmatrix},   
\end{equation}
where 
$X_{0, L-T+1}:=\left[\begin{array}{llll}
\bm{x}_d(0) & \bm{x}_d(1) & \ldots & \bm{x}_d(L-T)
\end{array}\right].$

Now consider a trajectory $[
\hat{\bm{u}}_{[0, T-1]}; 
\hat{\bm{x}}_{[0, T-1]}
]$ starting from an initial state $\hat{\bm{x}}(0)$ and evolves with the sequence of actions $\hat{\bm{u}}_{[0, T-1]}$. If $\begin{bmatrix}
\bm{U}_{0, T, L-T+1} \\
X_{0, L-T+1}
\end{bmatrix}$ is full row rank, namely $\text{rank}(\begin{bmatrix}
\bm{U}_{0, T, L-T+1} \\
X_{0, L-T+1}
\end{bmatrix})=n+Tm$, then there exists $\hat{\bm{g}}\in\real^{L-T+1}$ such that  $\begin{bmatrix}
\hat{\bm{u}}_{[0, T-1]} \\
\hat{\bm{x}}(0)
\end{bmatrix}=\begin{bmatrix}
\bm{U}_{0, T, L-T+1} \\
X_{0, L-T+1}
\end{bmatrix}\hat{\bm{g}}$. By~\eqref{eq: stack_transition},  
\begin{equation}
\begin{split}
\begin{bmatrix}
\hat{\bm{u}}_{[0, T-1]}\\ 
\hat{\bm{x}}_{[0, T-1]}
\end{bmatrix}
&=\begin{bmatrix}
\bm{I}_{T m} & \mathbbold{0}_{T m \times n} \\
\mathcal{T}_{[0,T-1]} & \mathcal{O}_{[0,T-1]}
\end{bmatrix}
\begin{bmatrix}
\bm{U}_{0, T, L-T+1} \\
X_{0, L-T+1}
\end{bmatrix}\hat{\bm{g}}\\
&=\begin{bmatrix}
\bm{U}_{0, T, L-T+1} \\
\bm{X}_{0, T, L-T+1}
\end{bmatrix}\hat{\bm{g}},
\end{split}
\end{equation}
where the second equation follows from the relation in~\eqref{eq:Hankel_u_x(k)ransition}. This complete the proof of Lemma~\eqref{lem:fundemental}.

\subsection{Policy gradient algorithm}\label{app: PG_state_feedback}

\begin{algorithm}[H]
 \caption{Policy Gradient with trajectory generation  }
 \label{alg: PG}
 \begin{algorithmic}[1]
 \renewcommand{\algorithmicrequire}{\textbf{Require: }}
 \renewcommand{\algorithmicensure}{\textbf{Data collection:}}
 \REQUIRE  The length $T$ of trajectory, the learning rate $\alpha$, total number of episode $I$\\
  \renewcommand{\algorithmicensure}{\textbf{Policy Gradient with Data generation:}}
 \ENSURE 
%  Generate a batch of $Q$ trajectories of the length $T$ subject to the control law $u(t)=\bm{\theta} x(t)+\bm{w}(t)$, $\bm{w}(t)\sim \mathcal{D}$.  \\
 \textit{Initialisation} :Initial weights $\bm{\theta}$ for control network
 \FOR {$episode = 1$ to $I$}
\STATE Generate a batch of  $Q$ trajectories $ \left[\bm{\tau}_1,\cdots,\bm{\tau}_Q \right]=$\texttt{TrajectoryGen}$\left(\mathcal{H},\bm{\theta},\mathcal{D},Q\right) $ 
 \STATE Compute the gradient $ \nabla J(\bm{\theta}) =  \frac{1}{Q}\sum_{i=1}^{Q} c(\bm{\tau}_i)\sum_{t=1}^{T} \nabla_{\bm{\theta}}\log\pi_{\bm{\theta}} (\Tilde{\bm{u}}_i(t)|\Tilde{\bm{x}}_i(t)) $
  \STATE  Update weights in the neural network by gradient descent:
  $\bm{\theta} \leftarrow \bm{\theta}-\alpha \nabla J(\bm{\theta}) $
  \ENDFOR
 \end{algorithmic} 
 \end{algorithm}

\subsection{Transition dynamics with extended states}\label{app: transition_partial}
% Let $T_0>\ell$ be the length of a trajectory. 
Expanding the transition dynamics in~\eqref{eq:sys_full} from time $0$ to $T_0$ gives
% \begin{equation}
%     \begin{split}
%         \bm{x}(1)&=A\bm{x}(0)+B\bm{u}(0)\\
%         \bm{x}(2) &= A^2\bm{x}(0)+AB\bm{u}(0)+B\bm{u}(1)\\
%         &\vdots\\
%         \bm{x}(T_0-1) &= A^{T_0-1}\bm{x}(0)+A^{T_0-2}B\bm{u}(0)\\
%         &\quad+\cdots+B\bm{u}(T_0-2)
%     \end{split}
% \end{equation}

% \begin{equation}
%     \begin{split}
%         \bm{y}(0)&=C\bm{x}(0)\\
%         \bm{y}(1) &= CA\bm{x}(0)+\bm{C}\bm{B}\bm{u}(0)\\
%         &\vdots\\
%         \bm{y}(T_0-1) &= CA^{T_0-1}\bm{x}(0)+CA^{T_0-2}B\bm{u}(0)\\
%         &\quad+\cdots+CB\bm{u}(T_0-2)
%     \end{split}
% \end{equation}
% Therefore
\begin{equation}\label{eq:y_u_traj_transition}
    \bm{y}_{[0, T_0-1]}=\mathcal{O}_{[0, T_0-1]}\bm{x}(0)+\mathcal{T}_{[0, T_0-1]}\bm{u}_{[0, T_0-2]}\end{equation}
where $\mathcal{T}_{[0, T-1]}$  and $\mathcal{O}_{[0, T-1]}$ are the Toeplitz and observability matrices of order $T_0$ represented as
$$
\begin{aligned}
\mathcal{T}_{[0, T_0-1]} &:=\left[\begin{array}{ccccc}
\bm{C}\bm{B}&\mathbbold{0}_{d\times m}&\mathbbold{0}_{d\times m}& \cdots &\mathbbold{0}_{d\times m}\\
\bm{C} \bm{A}\bm{B}& \bm{C}\bm{B}&\mathbbold{0}_{d\times m}& \cdots &\mathbbold{0}_{d\times m}\\
\vdots & \vdots & \vdots & \ddots & \vdots \\
\bm{C} \bm{A}^{T_0-2}\bm{B}& \bm{C}\bm{A}^{T_0-3}\bm{B}& \bm{C}\bm{A}^{T_0-4}\bm{B}& \cdots & \mathbbold{0}_{d\times m}
\end{array}\right] \quad
\mathcal{O}_{[0, T_0-1]} :=\left[\begin{array}{c}
\bm{C} \\
\bm{C}\bm{A} \\
\vdots \\
\bm{C}\bm{A}^{T_0-1}
\end{array}\right].
\end{aligned}
$$

Since  the system~\eqref{eq:sys_full} is observable,  $\mathcal{O}_{[0, T_0-1]}$ is full column rank. Thus,
\begin{equation}
    \bm{x}(0)=\underbrace{\left(\mathcal{O}_{[0, T_0-1]}^\top\mathcal{O}_{[0, T_0-1]}\right)^{-1}
    \mathcal{O}_{[0, T_0-1]}^{\top}}_{\mathcal{O}_{[0, T_0-1]}^\dagger} \left(\bm{y}_{[0, T_0-1]}-\mathcal{T}_{[0, T_0-1]}\bm{u}_{[0, T_0-2]}\right)
\end{equation}

% \LH{full column rank?}

% Then
% \begin{equation}
% \begin{split}
%      \bm{y}_{[1, t]}&=\mathcal{O}_{[1, t]}\bm{x}(0)+\mathcal{T}_{[1, t]}\bm{u}_{[0, t-1]}+\mathcal{M}_{[1, t]}\bm{u}(t)\\
%      &=\mathcal{O}_{[1, t]}\mathcal{O}_{[0, t-1]}^\dagger \left(\bm{y}_{[0, t-1]}-\mathcal{T}_{[0, t-1]}\bm{u}_{[0, t-1]}\right)\\
%      &\quad+\mathcal{T}_{[1, t]}\bm{u}_{[0, t-1]}+\mathcal{M}_{[1, t]}\bm{u}(t)\\
%      &= \mathcal{O}_{[1, t]}\mathcal{O}_{[0, t-1]}^\dagger \bm{y}_{[0, t-1]}+\mathcal{M}_{[1, t]}\bm{u}(t)\\
%      &\quad+\left(\mathcal{T}_{[1, t]}-\mathcal{O}_{[1, t]}\mathcal{O}_{[0, t-1]}^\dagger \mathcal{T}_{[0, t-1]} \right)\bm{u}_{[0, t-1]}
% \end{split}
% \end{equation}
% where 
% $$
% \begin{aligned}
% \mathcal{T}_{[1, t]} &:=\left[\begin{array}{ccccc}
% C B & D & 0 & \cdots & 0 \\
% C A B & C B & D & \cdots & 0 \\
% \vdots & \vdots & \vdots & \ddots & \vdots \\
% C A^{t-2} B & C A^{t-3} B & C A^{t-4} B & \cdots & D\\
% C A^{t-1} B & C A^{t-2} B & C A^{t-3} B & \cdots & CB\\
% \end{array}\right] \\
% \mathcal{O}_{[1, t]} &:=\left[\begin{array}{c}
% C A \\
% \vdots \\
% C A^{t-1}\\
% C A^{t}\\
% \end{array}\right]
% \end{aligned}
% $$
% \begin{equation*}
% \mathcal{M}_{[1, t]} &:=\left[\begin{array}{c}
% 0\\
% \vdots \\
% 0\\
% D\\
% \end{array}\right]    
% \end{equation*}

% Another version
Then pluging in the expression of $\bm{y}(T_0)$ yields
\begin{equation}\label{eq:y_T_0}
\begin{split}
     \bm{y}(T_0)&=\bm{C}\bm{A}^{T_0}\bm{x}(0)+\bm{C}\bm{A}^{T_0-1}\bm{B}\bm{u}(0)+\cdots+\bm{C}\bm{B}\bm{u}(T_0-1)\\
     &=\bm{C}\bm{A}^{T_0}\mathcal{O}_{[0, T_0-1]}^\dagger \left(\bm{y}_{[0, T_0-1]}-\mathcal{T}_{[0, T_0-1]}\bm{u}_{[0, T_0-2]}\right)+\mathcal{T}_{[T_0, T_0]}\bm{u}_{[0, T_0-2]}\\
     &= \bm{C}\bm{A}^{T_0}\mathcal{O}_{[0, T_0-1]}^\dagger \bm{y}_{[0, T_0-1]}+\left(\mathcal{T}_{[T_0, T_0]}-\bm{C}\bm{A}^{T_0}\mathcal{O}_{[0, T_0-1]}^\dagger \mathcal{T}_{[0, T_0-1]} \right)\bm{u}_{[0, T_0-2]}
\end{split}
\end{equation}
where 
$$
\begin{aligned}
\mathcal{T}_{[T_0, T_0]} &:=\left[\begin{array}{ccccc}
\bm{C}\bm{A}^{T_0-1} \bm{B} & \bm{C}\bm{A}^{T_0-2} \bm{B} & \bm{C}\bm{A}^{T_0-3} \bm{B} & \cdots & \bm{C}\bm{B}\\
\end{array}\right].
\end{aligned}
$$

Stacking the observations and the outputs together yields 
\begin{equation}
  \left[\begin{array}{c}
\bm{y}(1)\\
\vdots \\
\bm{y}(T_0)\\
\bm{u}(1)\\
\vdots \\
\bm{u}(T_0-1)\\
\end{array}\right]=\Tilde{\bm{A}} \left[\begin{array}{c}
\bm{y}(0)\\
\vdots \\
\bm{y}(T_0-1)\\
\bm{u}(0)\\
\vdots \\
\bm{u}(T_0-2)\\
\end{array}\right]
+\Tilde{\bm{B}}\bm{u}(T_0-1),
\end{equation}

where
\begin{equation}
\begin{split}
&\Tilde{\bm{A}}=
\begin{pmatrix}
  \begin{matrix}
\bm{0}& \bm{I}_d & \bm{0}& \bm{0}&\cdots & \bm{0}\\
\bm{0}& \bm{0}& \bm{I}_d & \bm{0}& \cdots & \bm{0}\\
\vdots & \vdots& \vdots & \vdots & \ddots & \vdots \\
\bm{0}& \bm{0} & \bm{0}& \bm{0}& \cdots & \bm{I}_d 
  \end{matrix}
  & \rvline & \bigzero \\
  \hline
  \bm{C}\bm{A}^{T_0}\mathcal{O}_{[0, T_0-1]}^\dagger& \rvline & \mathcal{T}_{[T_0, T_0]}-\bm{C}\bm{A}^{T_0}\mathcal{O}_{[0, T_0-1]}^\dagger \mathcal{T}_{[0, T_0-1]} \\
\hline
  \bigzero & \rvline &
  \begin{matrix}
\bm{0}& \bm{I}_m & \bm{0}& \bm{0}&\cdots & \bm{0}\\
\bm{0}& 0& \bm{I}_m & \bm{0}& \cdots & \bm{0}\\
\vdots & \vdots& \vdots & \vdots & \ddots & \vdots \\
\bm{0}& 0& \bm{0}& \bm{0}& \cdots & \bm{I}_m \\
  \end{matrix} \\
  \hline
  \bigzero & \rvline & \bigzero 
\end{pmatrix},
\Tilde{\bm{B}}=
\begin{pmatrix}
  \begin{matrix}
\bm{0}\\
\bm{0}\\
\vdots  \\
\bm{0}
  \end{matrix}\\
  \hline
 \bm{0}\\
\hline
  \begin{matrix}
\bm{0}\\
\bm{0}\\
\vdots  \\
\bm{0}
  \end{matrix}\\
  \hline
   \bm{I}_m 
\end{pmatrix}.    
\end{split}
\end{equation}

% and

% \begin{equation}
% \Tilde{\bm{B}}=
% \begin{pmatrix}
%   \begin{matrix}
% 0 \\
% 0 \\
% \vdots  \\
% 0 
%   \end{matrix}\\
%   \hline
%  0 \\
% \hline
%   \begin{matrix}
% 0 \\
% 0 \\
% \vdots  \\
% 0 
%   \end{matrix}\\
%   \hline
%   I
% \end{pmatrix}.    
% \end{equation}

\subsection{Policy Gradient for Extended States }~\label{app:Policy Gradient}
The basic policy gradient algorithm is~\citet{Sutton2018RL} 
% $\nabla_{\bm{\theta}}\mathbb{E}_\pi[c(\bm{\tau})] =\mathbb{E}_\pi[c(\bm{\tau}) \nabla_{\bm{\theta}}\log \pi(\bm{\tau})] $
\begin{equation}
\begin{split}
 \nabla_{\bm{\theta}} \mathbb{E}_{\bm{\tau} \sim p_{\bm{\pi}_{\bm{\theta}}}}[c(\bm{\tau})] &=\nabla_{\bm{\theta}} \int \pi_{\bm{\theta}}(\bm{\tau}) c(\bm{\tau}) d \tau \\
&=\int \nabla_{\bm{\theta}} \pi_{\bm{\theta}}(\bm{\tau}) c(\bm{\tau}) d \tau \\
&=\int \pi_{\bm{\theta}}(\bm{\tau}) \nabla_{\bm{\theta}} \log \pi_{\bm{\theta}}(\bm{\tau}) c(\bm{\tau}) d \tau \\
 &=\mathbb{E}_{\bm{\tau} \sim p_{\bm{\pi}_{\bm{\theta}}}}[c(\bm{\tau}) \nabla_{\bm{\theta}} \log \pi_{\bm{\theta}}(\bm{\tau})].  
\end{split}
\end{equation}

The probability of a trajectory with length $T$ is 
\begin{equation}\label{eq:prob_traj}
   \pi_\theta(\bm{\tau})= p\left(\mathcal{X}\left(T_0\right)\right) \prod_{t=T_0}^{T-1} p_{\pi_{\bm{\theta}}}\left(\bm{u}(t)|\mathcal{X}(t)\right) p\left(\mathcal{X}(t+1)|\mathcal{X}(t), \bm{u}(t)\right) .
\end{equation}

Expanding the terms in~\eqref{eq:prob_traj} and canceling the transition probability independent of $\bm{\theta}$, we have
\begin{equation}
\nabla_{\bm{\theta}}\log \pi_\theta(\bm{\tau}) =\sum_{t=0}^{T-1} \nabla_{\bm{\theta}}\log p_{\bm{\pi}_{\bm{\theta}}}\left(\bm{u}(k) \mid \bm{y}(k)\right),
\end{equation}
and therefore, 
\begin{equation}
    \nabla_{\bm{\theta}}\mathbb{E}_{\pi_\theta}[c(\bm{\tau})] =\mathbb{E}_{\bm{\tau} \sim p_{\bm{\pi}_{\bm{\theta}}}(\bm{\tau})}\left[c(\bm{\tau})\sum_{k=0}^{K-1} \nabla_{\bm{\theta}} \log p_{\bm{\pi}_{\bm{\theta}}}\left(\bm{u}(k) \mid \bm{y}(k)\right)\right]
\end{equation}

Hence, the policy gradient algorithm still holds for the output-feedback case.

\subsection{Proof of Theorem~\ref{thm: unque_trajectory0_partial}}\label{app:thm_partial}
To prove Theorem~\ref{thm: unque_trajectory0_partial}, we need to make use of the rank condition of $\bm{G}_{\bm{\theta}}$ and  $\bm{R}$ in~\eqref{eq:traj_g_partial}.  We first
show in Lemma~\ref{lem:null_partial} about the null space and rank condition induced from the condition $\text{rank}(\mathcal{H})=n+Tm$.
\begin{lemma}\label{lem:null_partial}
If the observability matrix  $\mathcal{O}_{[0, T_0-1]}$ is 
full column rank, then 
 the null space $\mathcal{N}(\bm{G}_{\bm{\theta}})$ is  the same as $\mathcal{N}(\mathcal{H})$. Moreover, if $\text{rank}(\mathcal{H})=n+Tm$, then  $\text{rank}(\bm{G}_{\bm{\theta}})=n+Tm$. 
\end{lemma}

\begin{proof}
We first prove that the null space $\mathcal{N}(\bm{G}_{\bm{\theta}})$ is  the same as $\mathcal{N}(\mathcal{H})$ from \textit{(i)}  and \textit{(ii)} :

\textit{(i)} For all $\bm{q}\in \mathcal{N}(\mathcal{H})$,  we have $[\mathcal{H}_y]\bm{q} =\mathbbold{0}_{Tn}$ and $[\mathcal{H}_u]\bm{q} =\mathbbold{0}_{Tm}$. Plugging in the expression of $\bm{G}_{\bm{\theta}}$ yields  $\bm{G}_{\bm{\theta}}\bm{q} = \mathbbold{0}_{Tm+n} $. Namely, $\bm{q}\in\mathcal{N}(\bm{G}_{\bm{\theta}})$.

\textit{(ii)} For all  $\bm{v}\in \mathcal{N}(\bm{G}_{\bm{\theta}})$,  we have
\begin{equation}
\left[\begin{array}{c}
\mathcal{H}_u^{T_0 -1: T-1} -\left(\bm{I}_{T-T_0}\otimes \theta\right) \mathcal{H}_y^{T_0 -1: T-1} \\
\mathcal{H}_y^{0 : T_0-1}\\
\mathcal{H}_u^{0 : T_0-2}
\end{array}\right]\bm{v}=\mathbbold{0}_{Tm+T_0 d} \; , 
\end{equation}
which gives
\begin{equation}\label{eq:fullC_null2}
\begin{split}
&\mathcal{H}_u^k \bm{v}=\theta\mathcal{H}_y^k \bm{v} \text{ for  } k=T_0,\cdots,T-1\\
&\mathcal{H}_y^{0 : T_0-1} \bm{v} =\mathbbold{0}_{T_0 d} \\
&\mathcal{H}_u^{0 : T_0-2} \bm{v} =\mathbbold{0}_{T_0 m} \;.  
\end{split}    
\end{equation}

From~\eqref{eq:y_T_0},  we have ,
\begin{equation}\label{eq:induction_y_u}
\begin{split}
   \mathcal{H}_y^{k}
   &=\bm{C}\bm{A}^{T_0}\mathcal{O}_{[0, T_0-1]}^\dagger\mathcal{H}_y^{k-T_0:k-1}+\left(\mathcal{T}_{[T_0, T_0]}-\bm{C}\bm{A}^{T_0}\mathcal{O}_{[0, T_0-1]}^\dagger \mathcal{T}_{[0, T_0-1]} \right)\mathcal{H}_u^{k-T_0:k-2} 
\end{split}
\end{equation}
for $k=T_0,\cdots,T-1$.

Plugging~\eqref{eq:fullC_null2} in~\eqref{eq:induction_y_u} induces 
$\mathcal{H}_y^k \bm{v} =\mathbbold{0}_{d}$ and $\mathcal{H}_u^k \bm{v} =\mathbbold{0}_{m}$ for $ k=0,\cdots,T-1$. Hence, $\mathcal{H}\bm{v} = \mathbbold{0}_{Tm+Td} $. Namely, $\bm{v}\in\mathcal{N}(\mathcal{H})$.

Next, we prove the rank condition. Note that $\mathcal{H}\in\real^{(Tm+Tn) \times (L-T+1)}$. If $\text{rank}(\mathcal{H})=n+Tm$, then the rank of Null space is $\text{rank}(\mathcal{N}(\mathcal{H}))=(L-T+1)-(n+Tm)$.  
Since $\mathcal{N}(\bm{G}_{\bm{\theta}})$ is  the same as $\mathcal{N}(\mathcal{H})$, then $\text{rank}(\mathcal{N}(\bm{G}_{\bm{\theta}}))=(L-T+1)-(n+Tm)$.  It follows directly that $\text{rank}(\bm{G}_{\bm{\theta}})=(L-T+1)-\text{rank}(\mathcal{N}(\bm{G}_{\bm{\theta}}))=n+Tm$. 
\end{proof}

Then, we are ready to prove Theorem~\ref{thm: unque_trajectory0_partial} as follows.

\begin{proof}
We first prove the existence of the solution in~\eqref{eq:traj_g_partial}. From~\eqref{eq:y_u_traj_transition},  we have
$$
\mathcal{X}(T_0-1):=\left[\begin{array}{l}
\bm{u}_{[0, T_0-2]} \\
\bm{y}_{[0, T_0-1]}
\end{array}\right]=\left[\begin{array}{cc}
I_{(T_0-1) m} & \mathbbold{0}_{(T_0-1) m \times n} \\
\mathcal{T}_{[0,T_0-1]} & \mathcal{O}_{[0,T_0-1]}
\end{array}\right]
\begin{bmatrix}
\bm{u}_{[0, T_0-2]} \\
\bm{x}(0)
\end{bmatrix}
$$
The number of element in the vector $[\bm{u}_{[0, T_0-2]}^\top, \bm{x}^\top(0)
]^\top$ is $n+m(T_0-1)$. Hence, the rank of $\mathcal{X}(T_0-1)$ is as most $n+m(T_0-1)$. Then the right side of~\eqref{eq:traj_g_partial} has the rank as most $n+m(T_0-1)+m(T-T_0+1)=n+Tm$. By Lemma~\ref{lem:null_partial},  $\text{rank}(\mathcal{H})=n+Tm$ yields $\text{rank}(\bm{G}_{\bm{\theta}})=n+Tm$. Hence, there exists at least one solution such that~\eqref{eq:traj_g0} holds.

Next, we show the uniqueness of the generated trajectory. Suppose there exists $\bm{g}_1$ and $\bm{g}_2$ are both solution of~\eqref{eq:traj_g0} and $\mathcal{H}\bm{g}_1\neq \mathcal{H}\bm{g}_2$. Since $\bm{g}_1$ and $\bm{g}_2$ are both solution of~\eqref{eq:traj_g0}, then $\bm{G}_{\bm{\theta}}\bm{g}_1=\bm{G}_{\bm{\theta}}\bm{g}_1$ and thus $\left(\bm{g}_1-\bm{g}_2\right)\in\mathcal{N}(\bm{G}_{\bm{\theta}})$.
On the other hand, $\mathcal{H}\bm{g}_1\neq \mathcal{H}\bm{g}_2$ yields $\mathcal{H}\left(\bm{g}_1-\bm{g}_2\right)\neq 0$ and thus $\left(\bm{g}_1-\bm{g}_2\right)\notin\mathcal{N}(\mathcal{H})$. This contradicts that $\mathcal{N}(\bm{G}_{\bm{\theta}})$ is  the same as $\mathcal{N}(\mathcal{H})$ proved in
Lemma~\ref{lem:null_partial}. Hence, $\mathcal{H}\bm{g}_1= \mathcal{H}\bm{g}_2$, namely, the generated trajectories are identical.
\end{proof}

\subsection{Trajectory generation algorithm for output-feedback control}\label{app:algorithm}

\begin{algorithm}[H]
 \caption{Trajectory generation for output-feedback control}
 \label{alg: TrajGenFull_partial}
 \begin{algorithmic}[1]
 \renewcommand{\algorithmicrequire}{\textbf{Require: }}
 \renewcommand{\algorithmicensure}{\textbf{Data collection:}}
%  \REQUIRE  The length $T$ of trajectory to be generated, the dimension of state $n$, the dimension of action $m$
 \ENSURE Collect historic measurement of the system and stack each $T$-length input-output trajectory  as Hankel matrix $\mathcal{H}$ shown in~\eqref{eq:Hankel_u_x} until  $\text{rank}(\mathcal{H})=n+Tm$ \\
  \renewcommand{\algorithmicensure}{\textbf{Data generation:}}
 \ENSURE 
 \textit{Input} :Hankel matrix $\mathcal{H}$, weights $\bm{\theta}$ and the distribution $\mathcal{D}$ for the control policy,  the batchsize $Q$ for the generated  trajectories, the set $\mathcal{S}_\mathcal{X}$ of historic initial extended state $\mathcal{X}(T_0-1) $ \\
     \SetKwFunction{FMain}{TrajectoryGen}
    \SetKwProg{Fn}{Function}{:}{}
    \Fn{\FMain{$\mathcal{H},\bm{\theta},\mathcal{D},Q ,\mathcal{S}_\mathcal{X} $}}{
          Plug in $\bm{\theta}$ to compute $\bm{G}_{\bm{\theta}}$ in~\eqref{eq:traj_g_partial}.\\
        Conduct eigenvalue decomposition of $\left(\bm{G}_{\bm{\theta}} \bm{G}_{\bm{\theta}}^\top\right)$ to obtain $\bm{P}_{\bm{\theta}}:=\begin{bmatrix}
        \bm{p}_1 & \bm{p}_2 &\cdots& \bm{p}_s
        \end{bmatrix}$ and $\bm{\Lambda}=\text{diag}(\lambda_1,\lambda_2,\cdots,\lambda_s)$ with $\lambda_i$ being nonzero eigenvalues and $\bm{p}_i$ being orthonormal eigenvectors.\\
        \For{$i = 1$ to $Q$}
        {Sample $\mathcal{X}(T_0-1)$ from $\mathcal{S}_\mathcal{X} $. Sample $\left\{\bm{w}_i(T_0-1),\cdots,\bm{w}_i(T-1)\right\}$ from distribution $\mathcal{D}$.\\
        Compute the  coefficient $\bm{g}_i^*=\bm{G}_{\bm{\theta}}^\top\bm{P}_{\bm{\theta}}\bm{\Lambda}^{-1}\bm{P}_{\bm{\theta}}^\top\left[
        \bm{w}_i(T_0-1)^\top 
        \cdots
        \bm{w}_i(T-1)^\top
        \mathcal{X}(T_0-1)^\top
        \right]^\top$.
        \\
        Generate the $i$-th trajectory    $\bm{\tau}_i:=\left[
        \Tilde{\bm{u}}_i(T_0-1)^\top 
        \cdots
        \Tilde{\bm{u}}_i(T-1)^\top
        \Tilde{\bm{y}}_i(T_0-1)^\top 
        \cdots
        \Tilde{\bm{y}}_i(T-1)^\top
        \right]^\top=\begin{bmatrix}
        \mathcal{H}_u^{T_0 -1: T-1}\\
        \mathcal{H}_y^{T_0 -1: T-1}
        \end{bmatrix}\bm{g}_i^*$.

        }
        \textbf{return} $ \left[\bm{\tau}_1,\cdots,\bm{\tau}_Q \right]$ 
        }
    % \textbf{End Function}
 \end{algorithmic} 
 \end{algorithm}

 \subsection{Control of a batch reactor system }\label{app:exp_reactor}
We use the discretized version of a batch reactor system in~\citep{de2019formulas} as an illustrative example. The state transition matrix is given in~\eqref{eq:example_AB}, where $\bm{x}(t)\in\real^4, \bm{u}(t)\in\real^2$. We aim to train a linear feedback controller to minimize  the cost $ J(\bm{\theta}) = \sum_{k=1}^{K}\left\| \bm{y}(k)\right\|_1 +0.1 \left\|\bm{u}_{\bm{\theta}}(k)\right\|_1$ for the trajectory in the time horizon $K=30$. The number of episode in training is $E=400$.
% Let $Q$ be the batchsize of trajectories for each episode of training. For policy gradient algorithm using sampled data on the real system, the total number of sampled data is $Q\times K \times E$. For testing, we randomly generate 800 initial states and compare the cost on trajectories starting from these states. 

\begin{equation}\label{eq:example_AB}
[A \mid B]=\left[\begin{array}{cccc|cc}
1.178 & 0.001 & 0.511 & -0.403 & 0.004 & -0.087 \\
-0.051 & 0.661 & -0.011 & 0.061 & 0.467 & 0.001 \\
0.076 & 0.335 & 0.560 & 0.382 & 0.213 & -0.235 \\
0 & 0.335 & 0.089 & 0.849 & 0.213 & -0.016
\end{array}\right]    
\end{equation}

\textbf{Case I: the state is directly observable.} We setup $T=K=30$ and collect historic trajectory of length $L = (m + 1)T -1 + n=93$. With a sampling period of 0.1s~\citep{de2019formulas}, the data collection takes 9.3s. Then we stack the data as a Hankel matrix~\eqref{eq:Hankel_u_x} and use Algorithm~\ref{alg: TrajGenFull} to generate trajectories with the batchsize 1200 in each training episode. 
% We compare the performance of REINFORCE policy gradient algorithm in Algorithm~\ref{alg: PG} using the generated trajectories (labeled as PG-TrajectoryGen) and the same algorithm using trajectories sampled on the true system (labeled as PG-Sample-Q for the sampling with the batchsize Q).
Figure~\ref{fig:Loss_AB_full_main}(a) and Figure~\ref{fig:Loss_AB_full_main}(b) shows the average batch loss along the training episodes and the testing loss, respectively. 
 PG-Sample-Q achieves both lower training loss and testing loss with larger batchsize Q. However, this improvement comes at the expense of an increased number of samples from the system, as shown in Figure~\ref{fig:Loss_AB_full_main}(c). By contrast, 
 PG-TrajectoryGen attains the same performance as  PG-Sample-1200, since Theorem~\ref{thm: unque_trajectory0} guarantees that we generate trajectories  as if they were truly sampled on the system. Moreover, the number of samples in PG-TrajectoryGen purely comes
from the fixed historic trajectory of the length $L=93$, which is even much smaller than PG-Sample-10 where the length of samples is 120000 (equals to $10\times30\times400$).

\textbf{Case II: only the first and the second element of the state is directly observed.} In this case, $\bm{y}(k)=(\bm{x}_1(k), \bm{x}_2(k))\in\real^2$. Then $T_0=2$ such that the observability matrix $\mathcal{O}_{[0, T_0]}(\bm{A}, \bm{C})$ is full column rank. According to the trajectory generation algorithm developed in Subsection~\ref{subsec: TrajGenPartial}, we setup  $T=K+T_0-1=31$ and collect historic measurements of length $L = (m + 1)T -1 + n=96$. With a sampling period of 0.1s, the data collection takes 9.6s. Hence, the data collection does not scale significantly even when only half of the state is directly observed. Figure~\ref{fig:Loss_AB_full_main}(a)-(c) compares the training loss, testing loss and the number of samples, respectively.  PG-TrajectoryGen achieves the same training and testing loss as PG-Sample-1000 with much smaller number of samples in the system.

% \begin{figure}[H]
% \centering
% \subfigure[Training Loss]{\includegraphics[width=1.8in]{figure/Loss_ABfull.png}}
% % \subfigure[Testing-FullObs]{\includegraphics[width=1.4in]{figure/Loss_ABfull_test.png}}
% \subfigure[Testing Loss]{\includegraphics[width=1.8in]{figure/Loss_ABfull_batchVSLoss.png}}
% \subfigure[Number of samples]{\includegraphics[width=1.8in]{figure/Loss_ABfull_BatchVSsample.png}}
% \caption{Performance of learning state-feedback controllers in the batch reactor system. PG-TrajectoryGen achieves the same training and testing loss as PG-Sample-1200 with much smaller number of samples on the system.}\label{fig:Loss_AB_full_main}
% \end{figure}

\begin{figure}[H]
\centering
\subfigure[Training Loss]{\includegraphics[width=1.8in]{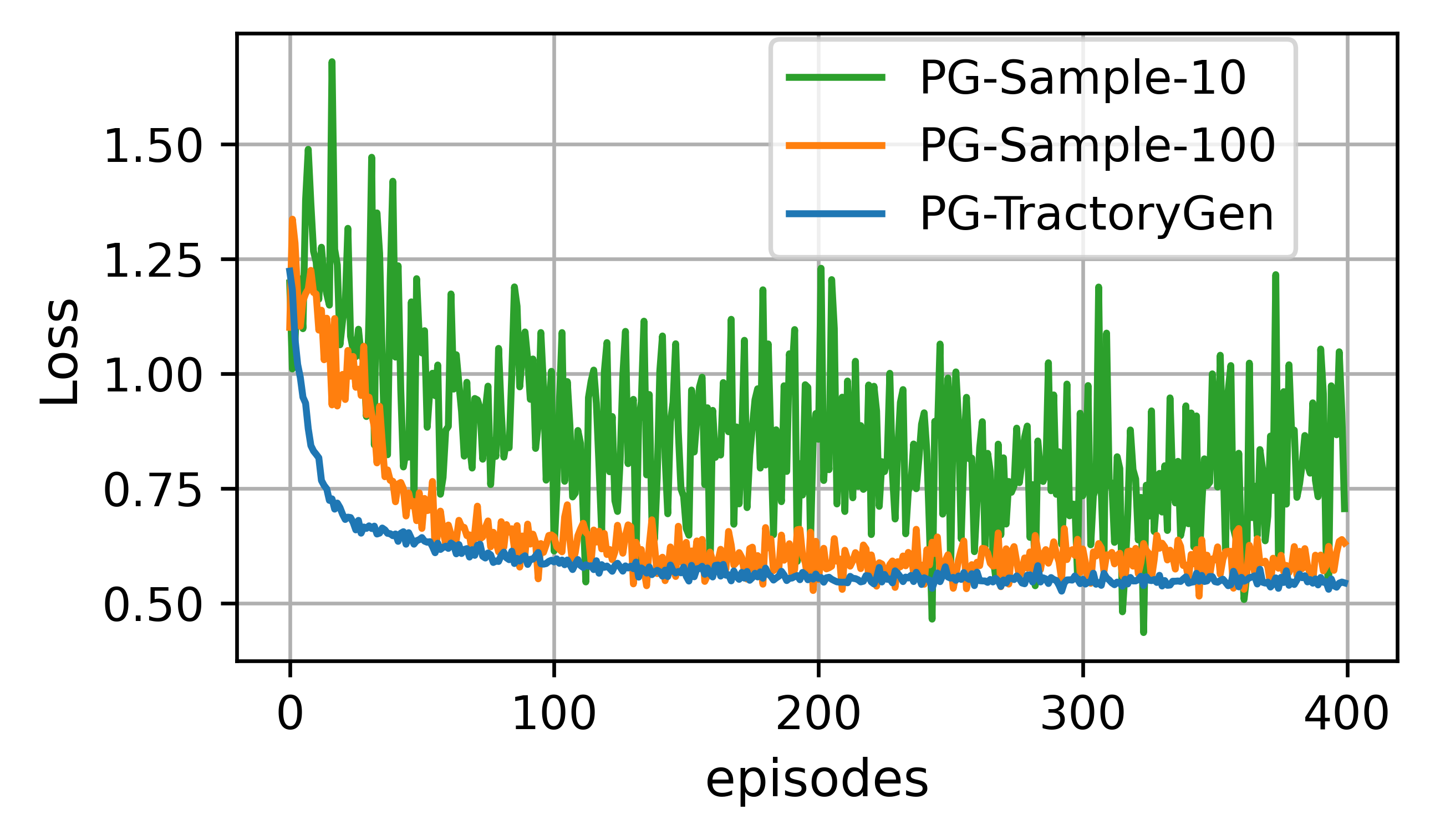}}
% \subfigure[Testing-FullObs]{\includegraphics[width=1.4in]{figure/Loss_ABfull_test.png}}
\subfigure[Testing Loss]{\includegraphics[width=1.8in]{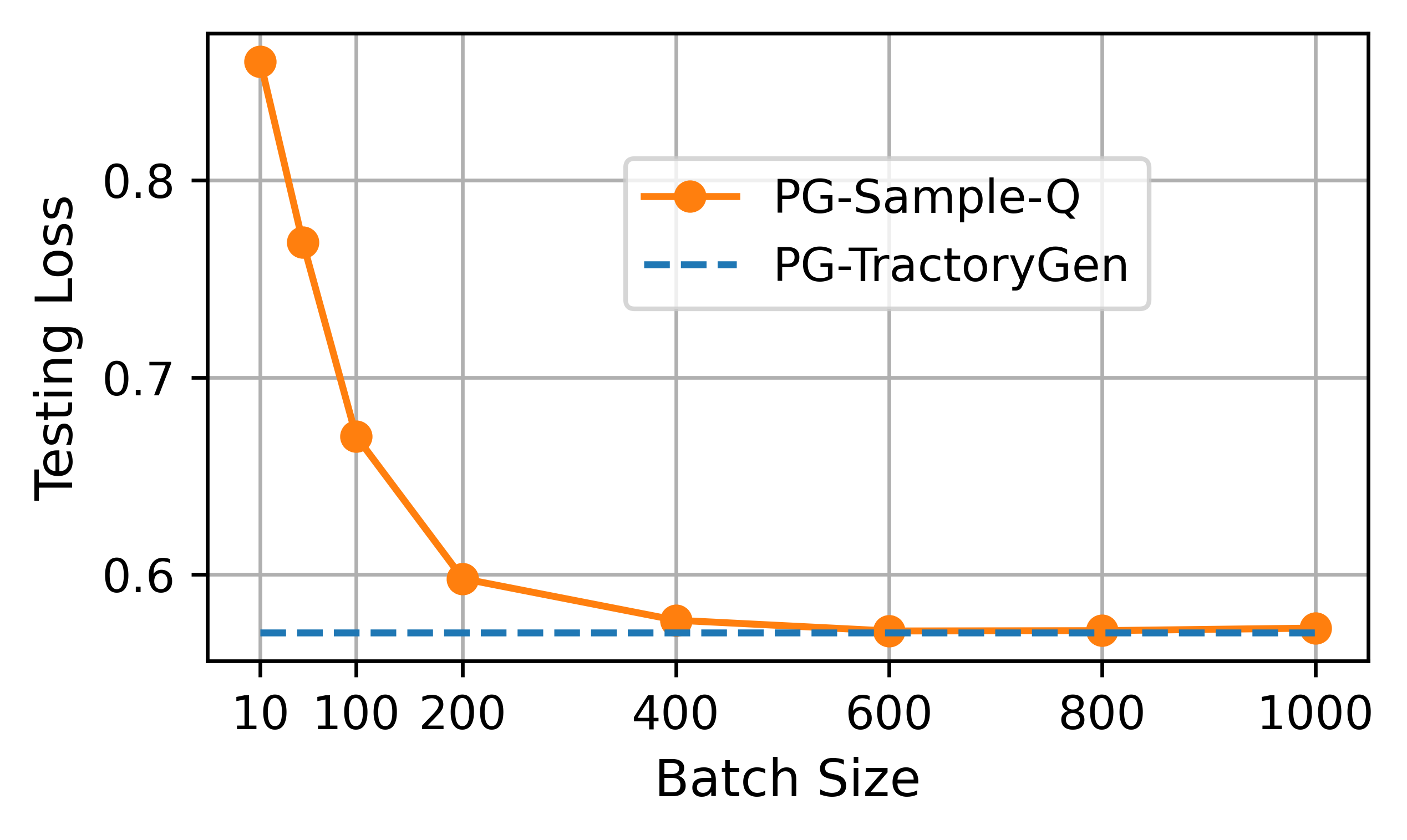}}
\subfigure[Number of samples]{\includegraphics[width=1.8in]{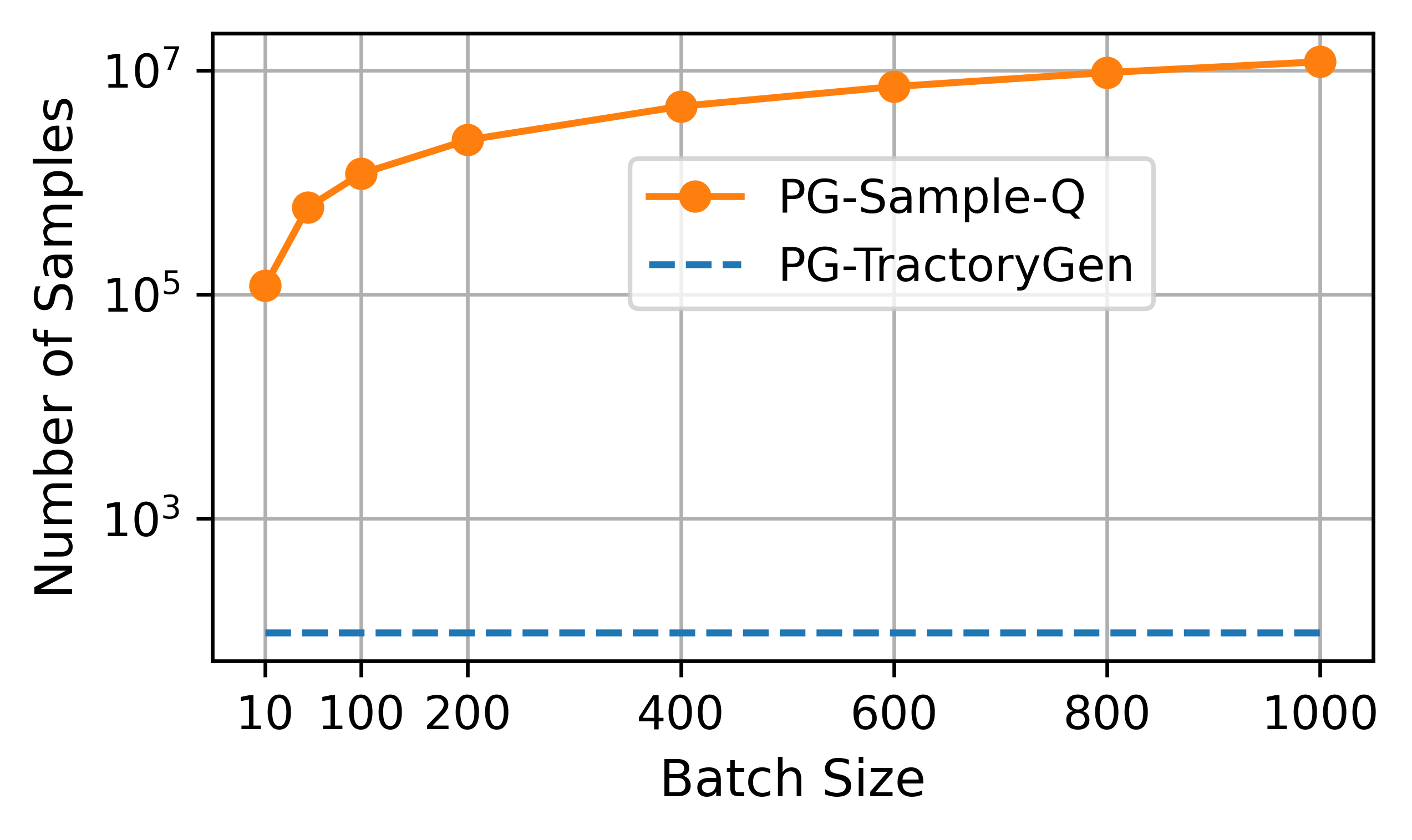}}
\caption{Performance of learning output-feedback controllers in the batch reactor system. PG-TrajectoryGen achieves the same training and testing loss as PG-Sample-1000 with much smaller number of samples on the system.}\label{fig:Loss_AB_partial}
\end{figure}

\subsection{Voltage control in power distribution networks}\label{app:exp_voltage}
To validate the performance of the proposed method on a larger system, we conduct experiments on the voltage control problem in IEEE 33bus test feeder~\citep{baran1989network}.  We adopt the Lindisflow model where the dynamics of voltage is described by a linear transition model~\citep{ cui2022decentralized, zhu2015fast}. The state $\bm{x}(t)\in\real^{32}$ is voltage in all the buses apart from the reference bus (the voltage of the reference bus is fixed). The action is the reactive power in each bus. We assume that there is no real-time communication between buses during real-time implementation, so the action at each bus can only change with the local measurement of voltage. The goal is to train a linear decentralized feedback controller to minimize total voltage deviation as well as the control effort in the time horizon $K=20$, written as $ J(\bm{\theta}) = \sum_{k=1}^{K}\left\| \bm{y}(k)\right\|_1 +0.3 \left\|\bm{u}_{\bm{\theta}}(k)\right\|_1$. The number of training episode is $E=500$.  

\textbf{Case I: the state is directly observed.} The state is directly observed so the action $\bm{u}(t)\in\real^{32}$. We setup $T=K=20$ and collect historic trajectory of the length $L = (m + 1)T -1 + n=691$. With a sampling period of 1s~\citep{chen2021enforcing}, the data collection takes 691s.
In each  episode of training, we use Algorithm~\ref{alg: TrajGenFull} to generate trajectories with the batchsize 1000. Figure~\ref{fig:Loss_voltage_full}(a)-(c) compares the training loss, testing loss and the number of samples, respectively.  PG-TrajectoryGen achieves the same training and testing loss as PG-Sample-1000 with much smaller number of samples on the system.

\textbf{Case II: only  $20$ elements in the state is observed. }  We assume only 20 buses are measured and controlled, so $\bm{y}(t),\bm{u}(t)\in\real^{20}$. The observability matrix $\mathcal{O}_{[0, T_0]}(\bm{A}, \bm{C})$ becomes full column rank when $T_0=3$. The time horizon of trajectory is $K=20$. According to the trajectory generation algorithm developed in Subsection~\ref{subsec: TrajGenPartial}, we setup $T=K+T_0-1=22$ and collect historic trajectory of length $L = (m + 1)T -1 + n=493$. With a sampling period of 1s~\citep{chen2021enforcing}, the data collection takes 493s. Figure~\ref{fig:Loss_voltage_partial_main}(a)-(c) compares the training loss, testing loss and the number of samples, respectively.  PG-TrajectoryGen achieves the same training and testing loss as PG-Sample-1000 with much smaller number of samples on the system.

% \begin{figure}[ht]
% \centering
% \subfloat[Training loss ]{\includegraphics[width=1.9in]{figure/Loss_voltage.png}%
% \label{fig_second_case}}
% \hfil
% \subfloat[Testing loss ]{\includegraphics[width=1.6in]{figure/Loss_voltage_test.png}%
% \label{fig_second_case}}
% \caption{\small (a) The average transient cost
% % over the time steps $t=0s, 0.02s,\cdots,6s$
% and steady-state cost
% % calculated at $t=20s$ 
% with error bar on the randomly generated  test set. (b) The dynamics of DenseNN-Comm. (c) The dynamics of  StableNN-Comm
% }
% \label{fig:Loss_epi_main}
% \end{figure}

\begin{figure}[H]
\centering
\subfigure[Training Loss]{\includegraphics[width=1.8in]{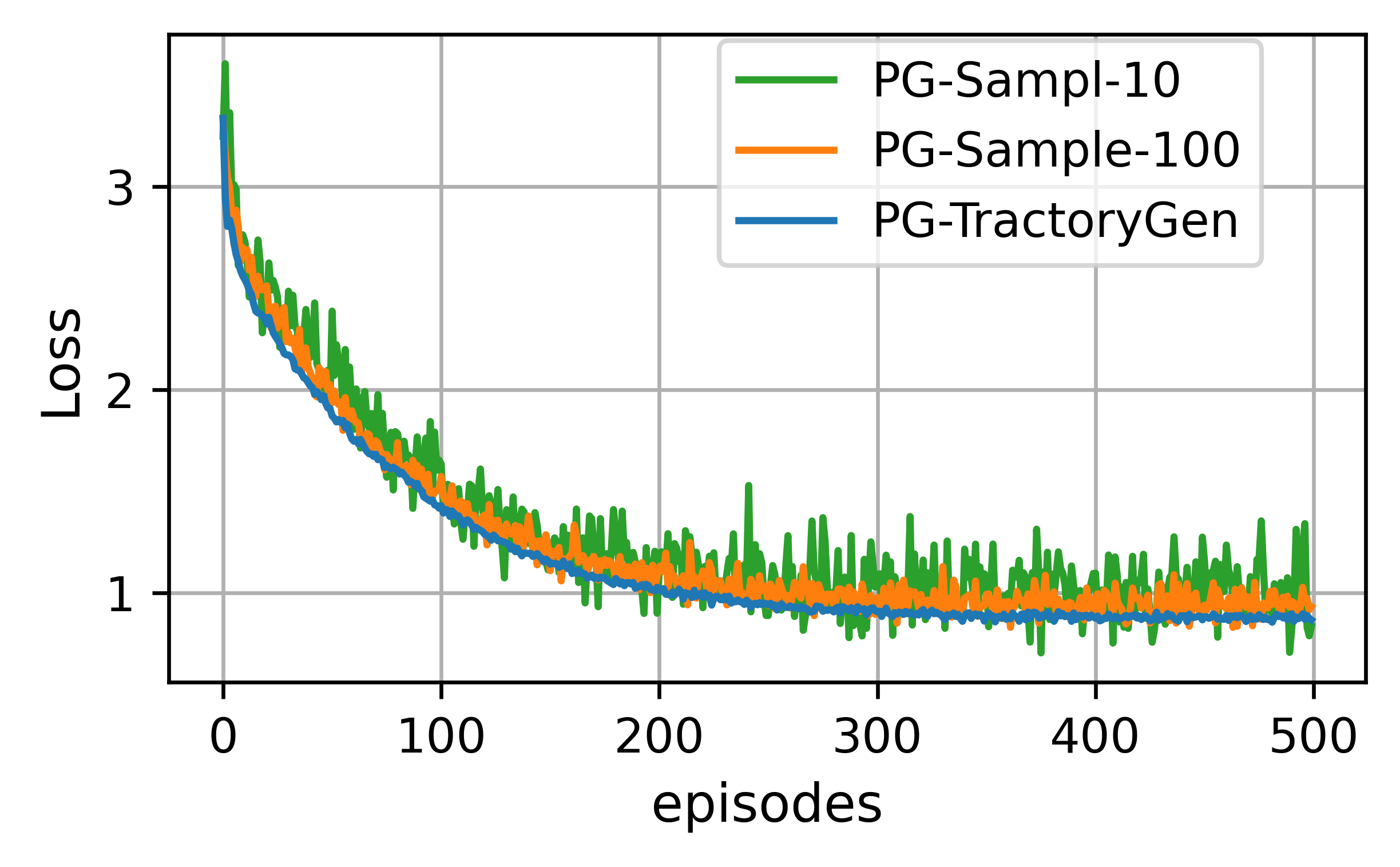}}
% \subfigure[Testing-FullObs]{\includegraphics[width=1.4in]{figure/Loss_ABfull_test.png}}
\subfigure[Testing Loss]{\includegraphics[width=1.8in]{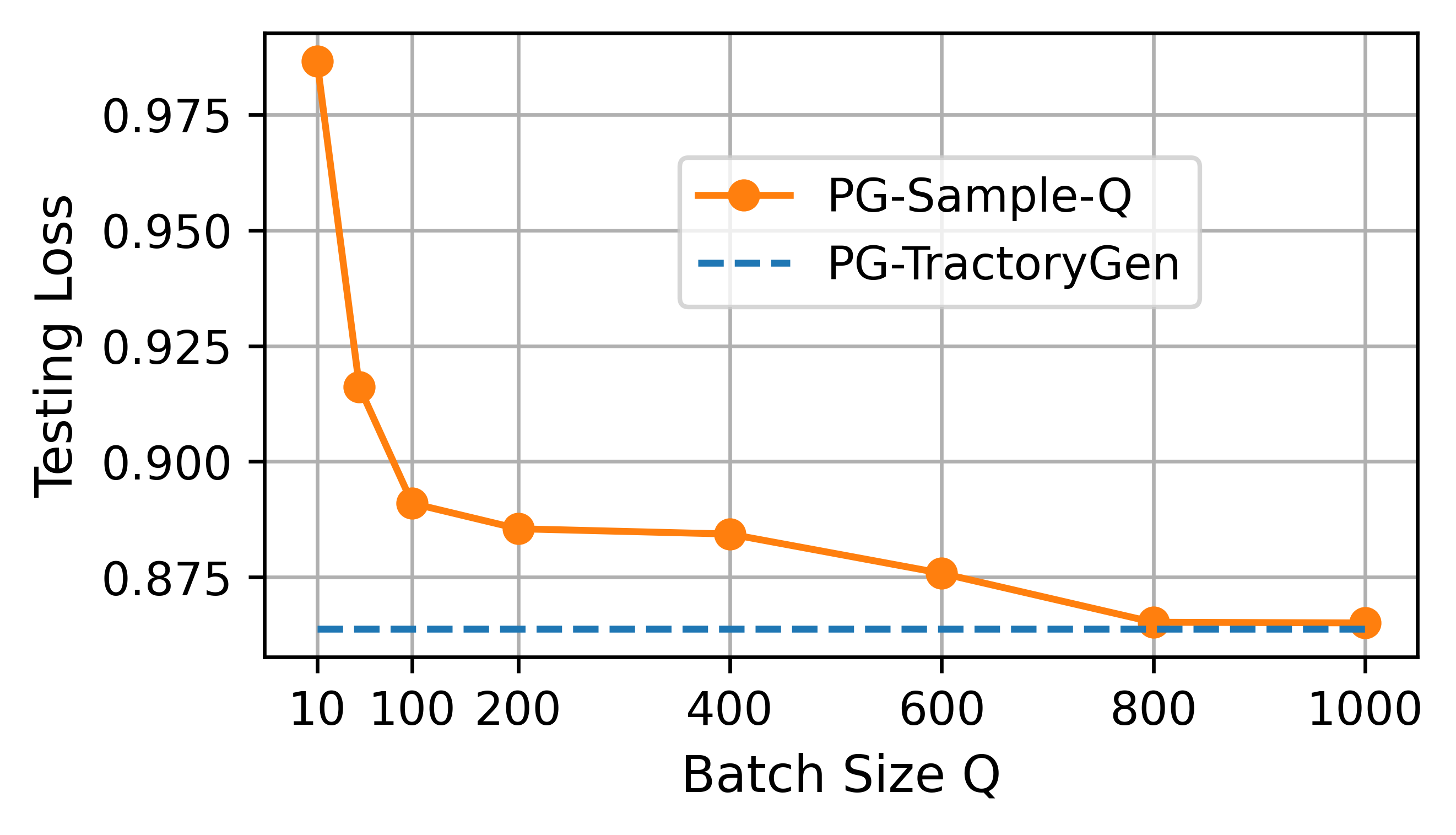}}
\subfigure[Number of samples]{\includegraphics[width=1.8in]{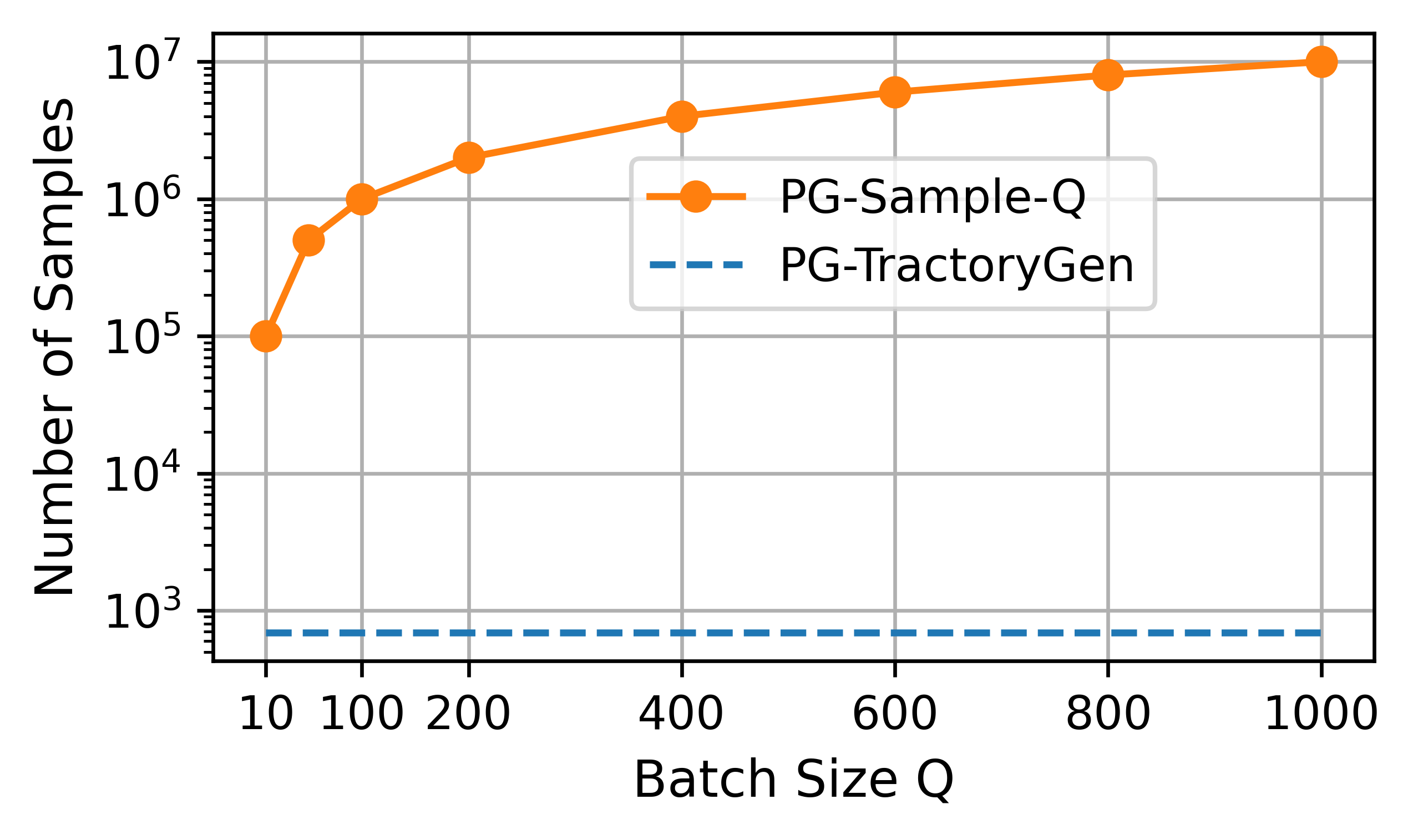}}
\caption{Performance of learning state-feedback controllers in power distribution network. PG-TrajectoryGen achieves the same training and testing loss as PG-Sample-1000 with much smaller number of samples on the system.}\label{fig:Loss_voltage_full}
\end{figure}

% \begin{figure}[H]
% \centering
% \subfigure[Training Loss]{\includegraphics[width=1.8in]{figure/Loss_voltage_partial.png}}
% % \subfigure[Testing-FullObs]{\includegraphics[width=1.4in]{figure/Loss_ABfull_test.png}}
% \subfigure[Testing Loss]{\includegraphics[width=1.8in]{figure/Loss_voltagePartial_batchVSLoss.png}}
% \subfigure[Number of samples]{\includegraphics[width=1.8in]{figure/Loss_voltagePartial_BatchVSsample.png}}
% \caption{Performance of learning output-feedback controllers in power distribution network. PG-TrajectoryGen achieves the same training and testing loss as PG-Sample-1000 with much smaller number of samples on the system.}\label{fig:Loss_voltage_partial}
% \end{figure}

\end{document}